\pdfoutput=1
\documentclass[11pt]{article}

\usepackage[preprint]{acl}

\usepackage{times}
\usepackage{latexsym}

\usepackage[T1]{fontenc}
\usepackage[utf8]{inputenc}

\usepackage{microtype}

\usepackage{inconsolata}

\usepackage{graphicx}

\usepackage{subfigure}
\usepackage{subcaption}
\usepackage[font=small,labelfont=bf]{caption}
\usepackage{booktabs} % for professional tables

\usepackage{hyperref}
\usepackage{multirow}
\usepackage{algorithm}
\usepackage{algorithmic}
\usepackage{amsmath}
\usepackage{amssymb}
\usepackage{mathtools}
\usepackage{amsthm}

\theoremstyle{plain}
\newtheorem{theorem}{Theorem}[section]

\newtheorem{lemma}[theorem]{Lemma}

\theoremstyle{definition}

\newtheorem{assumption}[theorem]{Assumption}
\theoremstyle{remark}

\newtheorem{principle}[theorem]{Principle}

\usepackage[textsize=tiny]{todonotes}

\usepackage{amsmath,amsfonts,bm}
\usepackage{amsthm}
\usepackage{amssymb}
\newcommand{\norm}[1]{\left\|#1\right\|}
\def\eqref#1{equation~\ref{#1}}
\def\1{\bm{1}}

\def\bnu{\boldsymbol{\nu}}
\DeclareMathAlphabet{\mathsfit}{\encodingdefault}{\sfdefault}{m}{sl}
\SetMathAlphabet{\mathsfit}{bold}{\encodingdefault}{\sfdefault}{bx}{n}

\def\gF{{\mathcal{F}}}

\def\bw{\boldsymbol{w}}
\def\bg{\boldsymbol{g}}
\def\bG{\boldsymbol{G}}
\def\bom{\boldsymbol{m}}

\newcommand{\E}{\mathbb{E}}

\newcommand{\Var}{\mathrm{Var}}

\newcommand{\btwo}{\beta_2}
\newcommand{\bu}{\boldsymbol{u}}

\def\bone{\beta_1}
\def\btwo{\beta_2}

\newcommand{\Pro}{\mathbb{P}}

\usepackage{dsfont}

\usepackage{color}
\definecolor{myfavblue}{rgb}{0.05, 0.2, 0.8}
\definecolor{keywords}{RGB}{255,0,90}
\definecolor{comments}{RGB}{0,0,113}
\definecolor{red}{RGB}{160,0,0}
\definecolor{green}{RGB}{0,150,0}
\definecolor{C0}{rgb}{0.12156862745098039, 0.4666666666666667, 0.7058823529411765}  % matplotlib C0

\definecolor{mydarkblue}{rgb}{0,0.08,0.45}
\hypersetup{
    colorlinks=true,
    linkcolor=mydarkblue,
    citecolor=mydarkblue,
    filecolor=mydarkblue,
    urlcolor=mydarkblue
}
\newcommand*{\nameA}[1]{{\emph{AdamS}}}
\newcommand*{\fullname}[1]{{\emph{Adam with momentum as a Self normalizer}}}
\newcommand*{\nameS}[1]{{\emph{MES}}}
\title{\nameA{}: Momentum Itself Can Be A Normalizer \\for LLM Pretraining and Post-training\thanks{The first two authors contribute equally. \textbf{Correspondence to:} \emph{zhanghuishuai@pku.edu.cn, bhwangfy@gmail.com, cpchenpi@mail.ustc.edu.cn}.
 }}

\author{
 \textbf{Huishuai Zhang\textsuperscript{1,3$\dagger$}},
 \textbf{Bohan Wang\textsuperscript{2$\dagger$}},
 \textbf{Luoxin Chen\textsuperscript{2}},
 \\
\\
 \textsuperscript{1}Wangxuan Institute of Computer Technology, Peking University\\
 \textsuperscript{2} University of Science and Technology of China\\
 \textsuperscript{3}State Key Laboratory of General Artificial Intelligence
}

\begin{document}
\maketitle

\begin{abstract}

We introduce \nameA{}, a simple yet effective alternative to Adam for large language model (LLM) pretraining and post-training. By leveraging a novel denominator, i.e., the root of weighted sum of squares of the momentum and the current gradient, \nameA{} eliminates the need for second-moment estimates. Hence, \nameA{} is efficient, matching the memory and compute footprint of SGD with momentum while delivering superior optimization performance. 
Moreover, \nameA{} is easy to adopt: it can directly inherit hyperparameters of AdamW, and is entirely model-agnostic, integrating seamlessly into existing pipelines without modifications to optimizer APIs or architectures. The motivation behind \nameA{} stems from the observed $(L_0, L_1)$ smoothness properties in transformer objectives, where local smoothness is governed by gradient magnitudes that can  be further approximated by momentum magnitudes. We establish rigorous theoretical convergence guarantees and provide practical guidelines for hyperparameter selection. 
Empirically, \nameA{} demonstrates strong performance in various tasks, including pre-training runs on GPT-2 and Llama2 (up to 13B parameters) and  reinforcement learning in post-training regimes. With its efficiency, simplicity, and theoretical grounding, \nameA{} stands as a compelling alternative to existing optimizers. The code is available at https://github.com/pku-huzhang/AdamS.
\end{abstract}

\begin{figure}[htb]
\begin{center}
\begin{minipage}[t]{0.49\linewidth}
\centering
{\includegraphics[width=\linewidth]{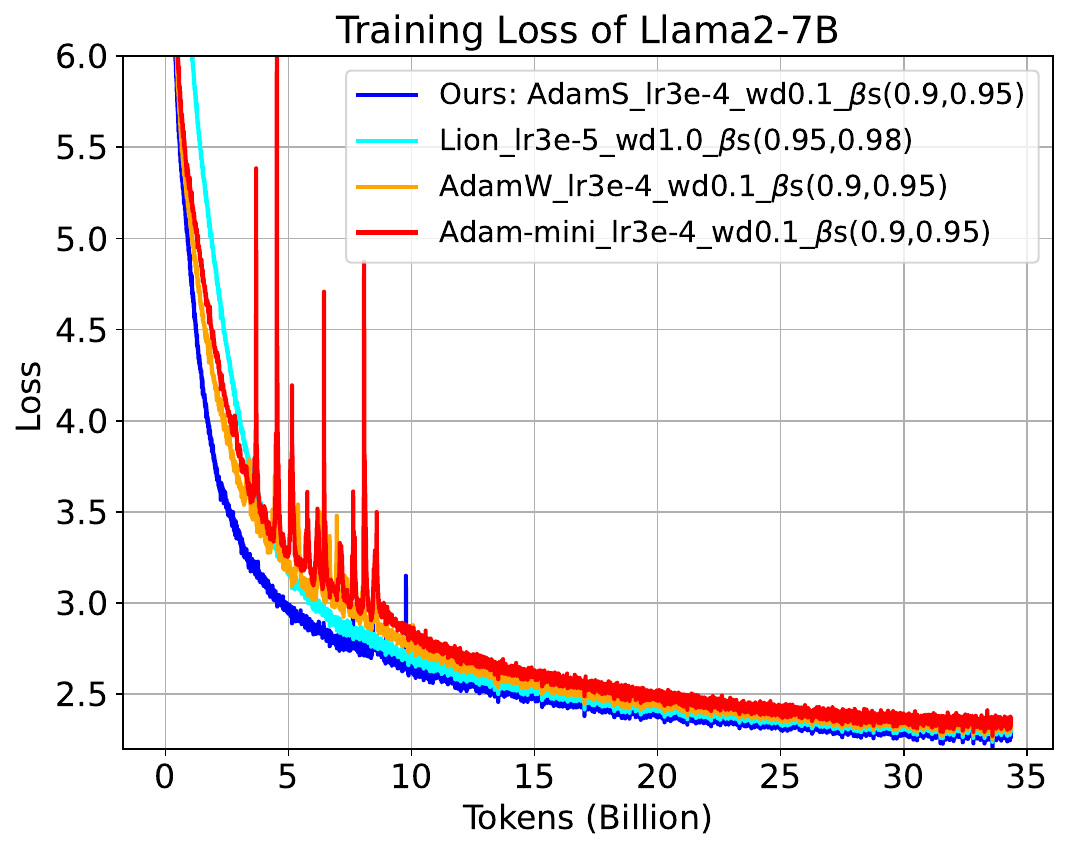}} 
 % (a) GPT2-small (124M)
\end{minipage}
\hfill
\begin{minipage}[t]{0.49\linewidth}
\centering
{\includegraphics[width=\linewidth]{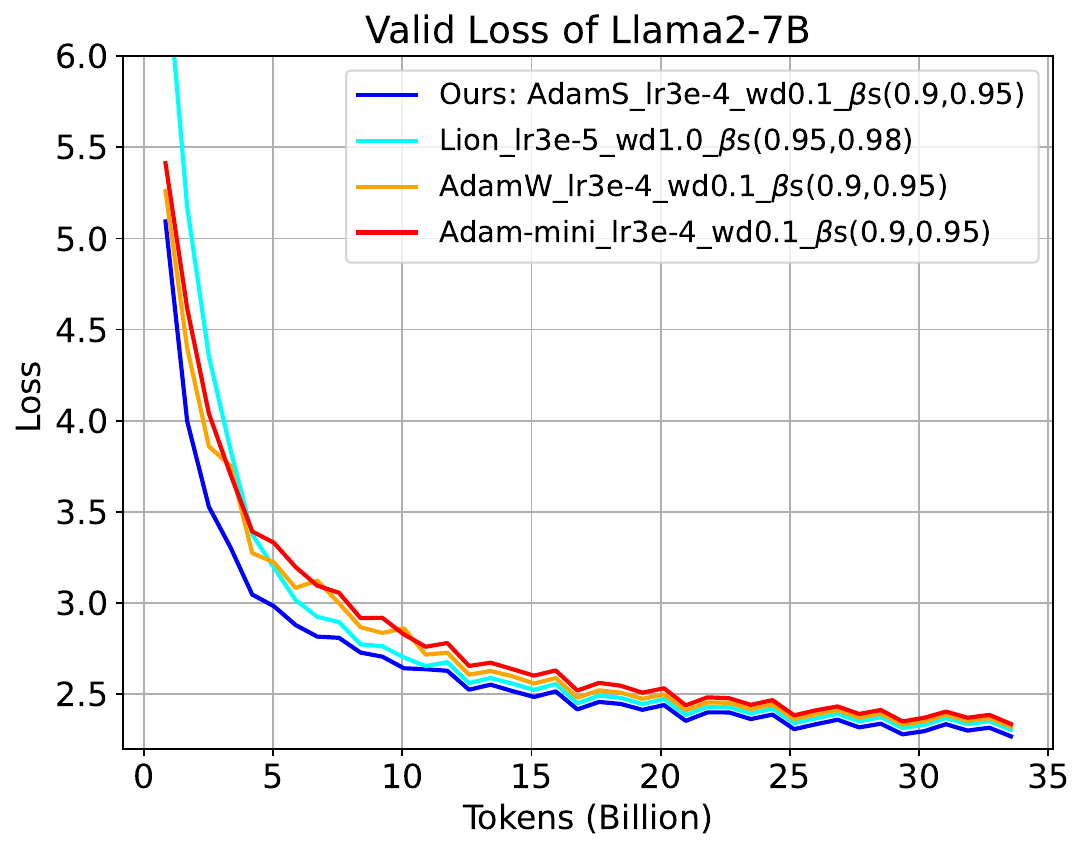}} 
 % (b) GPT2-base (350M) 
\end{minipage}

\end{center}
\caption{
Training and validation loss curves for pretraining LLaMA 2–7B models. The proposed \nameA{} achieves convergence comparable to or better than baseline methods under the same hyperparameter settings as LLaMA 2~\cite{touvron2023llama2}, while eliminating the need to store AdamW’s second-moment estimates.
}
\label{fig:llama}
\end{figure}

\section{Introduction}

Due to the scaling law \citep{kaplan2020scaling} of neural networks, it has been enthusiastic in the AI community to pre-train large foundation models with enormous data over the past years \citep{touvron2023llama,brown2020gpt3,zhang2022opt,rae2021scaling,chowdhery2022palm,du2021glm,liu2024deepseek,dubey2024llama,yang2024qwen2}. 
Training such large foundation models become super challenging because of tremendous engineering efforts,  computational cost \citep{deepspeed,guo2025deepseek}, and potential training spikes~\citep{zhang2022opt,molybog2023theory,chowdhery2022palm}. 

One reason for such high cost comes from the widely used optimizer \emph{Adam}~\citep{kingma2014adam} or \emph{AdamW}~\citep{loshchilov2019adamw}: the optimizers require storing both the state of momentum and the state of second-moment estimates, which  consumes $2\sim 4$ times GPU memories of the model size,  huge for models with hundreds of billions of parameters. In practice, practitioners employ advanced distributed‐training frameworks, such as Fully Sharded Data Parallel (FSDP)~\citep{zhao2023pytorch} and  DeepSpeed’s ZeRO optimizer~\citep{deepspeed}, to shard optimizer state across multiple GPUs and exchange only the necessary parameters over high‐bandwidth interconnects, thereby compensating memory consumption by communication.

In this paper, we try to reduce such memory cost  by proposing a simple yet effective  optimizer \nameA{}, an alternative to AdamW.  \nameA{} eliminates the need for second-moment estimates,  by leveraging a novel denominator: the root of weighted sums of  squares of the momentum and the current gradient. As a consequence, \nameA{} matches the memory and compute footprint of stochastic gradient descent (SGD) with momentum while delivering superior performance as good as AdamW. 

The design of \nameA{} is inspired by the observation that transformer-based models, which dominate modern large language models (LLMs), exhibit unique smoothness properties in their optimization landscapes. Specifically, the local smoothness of these objectives is governed by gradient magnitudes, which suggests that the learning rate should be proportional to the reciprocal of the gradient norm at each iteration, explaining why Adam optimizer beats SGD on training transformer-like architectures \citep{zhang2019gradient,wang2023convergence}. We further employ the fact that momentum, an exponential average of historical gradients, can provide a good and robust estimate of gradient magnitude \citep{cutkosky2020momentum,zhang2020improved} without the need for complex second-moment computations. By leveraging this insight, \nameA{} reduces memory cost of the optimizer states by half. Such efficiency of \nameA{} is particularly attractive for large-scale training, where even small improvements in efficiency can translate into significant cost savings. 

% We note that there has been effort of designing new optimizers either for less memory cost, i.e., Adafactor \citep{shazeer2018adafactor}, Adam-mini \citep{zhang2024adam}, Shampoo \citep{gupta2018shampoo}, Lion \citep{chen2023symbolic}, or for better convergence, i.e,  Sophia \citep{liu2023sophia}, NAdam \citep{dozat2016nadam}, AdaBound \citep{luo2019adabound}, AdaBelief \citep{zhuang2020adabelief}, and RAdam \citep{liu2020radam}, Adam \citep{kingma2014adam} and its variant AdamW \citep{loshchilov2019adamw} remain the dominant choices in both academic and industrial deep learning implementations \citep{schneider2021HITY}. This reluctance to adopt new optimizers stems from the difficulty of constantly surpassing AdamW in large-scale learning \citep{kaddour2023no} and the fundamental role that optimizers play in training. Practitioners are hesitant to switch unless a new optimizer has clear advantages, is easy to tune, and integrates seamlessly into existing workflows.

Recognizing this deep-rooted dependency on AdamW, we emphasize that \nameA{} is easy to adopt and can serve as a drop-in replacement for AdamW for pre- and post-training tasks of LLM. Moreover, \nameA{} is  model-agnostic, making it easy to integrate into existing pipelines without modifications to  APIs or model architectures. More importantly, it inherits AdamW’s hyperparameter configuration, thereby mitigating the often prohibitive costs of hyperparameter re-tuning and minimizing the risk associated with deploying a new optimizer at scale.

Empirically, \nameA{} demonstrates strong performance across a wide range of tasks and architectures, namely the transformer-based next-token prediction pretraining tasks and GRPO reinforcement learning tasks. In pretraining scenarios, it matches or exceeds the performance of AdamW on models ranging from GPT-2 to Llama2, with parameter counts up to 13B as shown in Figure~\ref{fig:llama}. This scalability is particularly important given the growing trend toward even larger models and datasets. Additionally, \nameA{} excels in post-training tasks, including reinforcement learning (RL), where it achieves state-of-the-art results in tasks such as the DeepSeek R1-Zero replication. This versatility underscores its potential as a general-purpose optimizer for both pretraining and post-training paradigms.

On the theoretical side, we establish rigorous convergence guarantees that demonstrate the effectiveness of \nameA{} in optimizing non-convex objectives, which are typical in LLM training. These guarantees are derived under realistic assumptions about the smoothness and noise properties of the optimization landscape.

Our contributions can be summarized as follows:
\begin{itemize}
    \item \textbf{Innovative Optimizer Design}: We introduce \nameA{}, which eliminates the need for second-moment estimates by leveraging a novel normalization strategy based on a weighted momentum-gradient combination. This approach  reduces the memory footprint of optimizers' state by  50\% while maintaining the ease of adoption.
    
    \item \textbf{Theoretical Grounding}: We rigorously analyze the convergence guarantees of \nameA{} for optimizing non-convex objectives under relaxed smoothness and weak noise assumptions, which matches the lower bounds of any gradient-based  optimizers.

    \item \textbf{Empirical Validation}: Through extensive experiments, e.g., large-scale pretraining on models like GPT-2 and Llama2 (up to 13B parameters) and reinforcement learning post-training tasks such as DeepSeek R1-Zero replication, we demonstrate that \nameA{} consistently matches AdamW, underscoring its versatility across different training paradigms.
\end{itemize}

In the following sections, we detail the motivation and formulation of \nameA{}. We then present the theoretical analysis and convergence guarantees, followed by an extensive empirical study spanning a variety of tasks and architectures. Through this comprehensive exploration, we aim to establish \nameA{} as a compelling alternative in the evolving landscape of large language model pretraining and post-training optimization.

\subsection{Related Works}

\textbf{The smoothness property of transformer-like architectures.}  The seminal work \citep{zhang2019gradient}  introduced the $(L_0,L_1)$-smooth condition that assumes local smoothness bounded by the local gradient norm, which is nicely verified by the optimization landscape of training transformer-like models.
Under these assumptions, convergence properties of adaptive optimizers, AdaGrad \citep{faw2023beyond,wang2023convergence}, Adam \citep{wang2022provable,he2023convergence,wang2023convergence,li2023convergence} are established and the benefit over SGD is demonstrated.  Our design of \nameA{} is inspired by these local smoothness properties, and delivers robust empirical performance, where gradient magnitudes govern optimization dynamics particularly in transformer-like architectures. 

\textbf{Memory-efficient adaptive learning rate optimizers.} In the development of memory-efficient adaptive learning rate optimizers, several notable methods have been proposed to address the challenges of high memory consumption in large-scale neural network training. \citet{shazeer2018adafactor} introduced Adafactor,  which reduces memory usage by maintaining only per-row and per-column sums of the second-moment estimates for weight matrices. 
\citet{anil2019memory} proposed SM3, a memory-efficient adaptive optimization method that approximates second-moment statistics with sublinear memory cost by partitioning parameters and sharing second-moment estimates among them. SM3 achieves per-parameter adaptivity with reduced memory overhead, facilitating the training of larger models and mini-batches. \citet{luo2023came} developed CAME to address the instability issues of existing memory-efficient optimizers via a confidence-guided adaptive strategy.  \citet{lv2023adalomo} introduced AdaLomo, which combines low-memory optimization techniques with adaptive learning rates by employing non-negative matrix factorization for second-order moment estimation. \citet{zhao2024galore} proposed GaLore that projects weight gradients onto a low-rank subspace, and update the model in the low-rank subspace, enabling fine-tuning LLM with consumer-grade GPUs with 24GB memory, where the idea of low-rank projection has been initiated in \citep{yu2021large}. Recently, \citet{zhang2024adam} proposed Adam-mini, an optimizer that reduces memory usage by partitioning model parameters into blocks based on the Hessian structure and assigning a single learning rate to each block, reducing memory consumption of optimizer state by approximately 45\% to 50\%.

Despite the proliferation of all these advancements, 
practitioners often hesitate to move away from AdamW 
because they either need to tune more hyperparameters, or require to be aware of the model architecture, or do not systematically surpassing AdamW in large-scale learning \citep{kaddour2023no, hoffmann2022training}. In contrast, \nameA{}  offers a model-agnostic solution that seamlessly integrates into existing workflows.  It requires no additional hyperparameters beyond those used in AdamW, allowing for straightforward adoption and tuning. Moreover, \nameA{} matches the memory efficiency of vanilla SGD with momentum while delivering performance comparable to AdamW, making it a practical drop-in replacement that one can enjoy benefits with minimal effort.

Adam-mini indeed targets memory efficiency, but it requires architectural awareness (e.g., grouping parameters), whereas AdamS applies in a model-agnostic way, without model-specific modifications. Adam-mini also maintains a second-moment approximation, albeit coarsely, while AdamS eliminates it entirely.

The main claim of Adam-mini paper is that Adam-mini can mimic the performance of AdamW with memory saving of the second moments. Hence it is sufficient to compare AdamS with AdamW given the performance of Adam-mini is fully captured by AdamW.

\section{Motivation and Design Choices of \nameA{}}

This section outlines the motivation behind our optimizer design—specifically, the rationale for adopting the root mean square of a properly weighted momentum itself and the current gradient as an adaptive denominator. We then formalize the algorithm and analyze its properties.

\subsection{Motivation and  $(L_0,L_1)$ smoothness}

In classical optimization settings, gradient descent provably decreases the loss at each iteration—provided the learning rate is smaller than the inverse of the smoothness constant. However, this principle fails to hold for transformer-based models, where stochastic gradient descent (SGD) with momentum exhibits poor convergence empirically. Recent work \cite{zhang2019gradient} identifies a key observation: Transformer training objectives violate standard smoothness assumptions and instead obey a relaxed \((L_0, L_1)\)-smoothness condition. Under this regime, the local smoothness depends on the gradient magnitude, enabling pathological curvature that can arbitrarily slow SGD’s progress \cite{wang2023closing}.  The \((L_0, L_1)\)-smoothness assumption  is as follows.

\begin{assumption}[$(L_0,L_1)$-smooth condition]
\label{assum: objective}
Assuming that $f$ is differentiable and lower bounded, there exist  constants $L_0,L_1>0$, such that $\forall \bw_1, \bw_2 \in \mathbb{R}^d$ satisfying $\Vert \bw_1 -\bw_2 \Vert \le \frac{1}{L_1}$,
\begin{flalign*}
    &\Vert \nabla f(\bw_1) -\nabla f(\bw_2) \Vert \\
    \le &(L_0+L_1 \Vert \nabla f(\bw_1) \Vert)\Vert \bw_1 -\bw_2 \Vert.
\end{flalign*}
    
\end{assumption}

Assumption \ref{assum: objective} is a general form of $(L_0,L_1)$-smooth condition, equivalent to the Hessian-bound form \cite{zhang2019gradient} when Hessian exists. 

When Assumption~\ref{assum: objective} holds, the local smoothness of the objective function is bounded by the the linear form of the gradient norm (i.e., \( L(\bw) \leq L_0+L_1 \|\nabla f(\bw)\| \). We know that the \emph{smoothness constant \( L(\bw) \)} governs how much the gradient can change locally. If \( L(\bw) \) scales with \( \|\nabla f(\bw)\| \), the curvature (and thus the risk of overshooting) increases with the gradient's magnitude. This necessitates a smaller learning rate when the gradient is large and allows a larger rate when the gradient is small.

A brief derivation (see details in Appendix \ref{app:descent}) gives a range of $\eta_t$ that guarantees decreasing function value at each step, i.e., $\eta_t \le 1/(L_0+L_1 \|\nabla f(\bw_t)\|)$, 
which ensures convergence by balancing the descent and curvature terms. This adaptively scales \( \eta \) inversely with the grad's magnitude.

In practice, we do not know the exact values of  $L_0$ and $L_1$, a typical choice of $\eta_t$ should be 
\[
\eta_t = \frac{C}{ \|\nabla f(\bw_t)\|+\epsilon},
\]
for some constant or scheduled constant $C$ after taking account of avoiding explosion near minima. Such an argument can be extended to coordinate-wise sense, which necessitates  per-coordinate adaptive learning rates.

We note that Adam adapts learning rates using second-moment estimates, i.e., the exponential average of of the square of historical gradients to approximate the gradient magnitude.  We draw inspiration from \cite{zhang2020improved}, which demonstrates that momentum—the exponential moving average of historical gradients—can itself serve as a robust proxy for gradient magnitudes. Building on this insight, we propose replacing second-moment estimation with a novel denominator derived from a weighted combination of momentum and the current mini-batch gradient. This approach retains the benefits of adaptive learning rate tuning while eliminating the computational overhead of tracking second moment statistics.

\subsection{The Design of \nameA{}}

The design of \nameA{} is given by Algorithm \ref{alg:adams}. Specifically, the denominator is 
$$\bnu_{t}\leftarrow \beta_{2}\bom_{t-1}^{\odot 2}+ (1-\beta_{2})\bg_t^{\odot 2}.$$
\begin{algorithm}
    \caption{ {\color{orange} AdamW} v.s. {\color{blue}\nameA{}}}\label{alg:adams}
    \begin{algorithmic}[1]
    \STATE \textbf{Input:} momentum parameter $\beta_1$, denominator parameter $\beta_2$, weight decay $\lambda$, learning rate $\eta$, objective  $f$, regularizer $\epsilon$
    \STATE \textbf{Initialize:}  $\bw_0$,   $\bom_0\leftarrow 0,\bnu_0 \leftarrow 0, t\leftarrow 0$ 
     \WHILE{$\bw_t$ not converged}
        \STATE $t \leftarrow t + 1$
        \STATE $\bg_t \leftarrow \nabla_{\bw}{f(\bw_{t-1})}$
        \STATE \textbf{update state tracking}
        \STATE $\bom_t \leftarrow \beta_1 \bom_{t-1} + (1 - \beta_1)\bg_t$
        
        % \STATE $\bnu_{t}\leftarrow \beta_{\nu}\bnu_{t-1}+ (1-\beta_{\nu})\bg_t^{\odot 2}$
          \STATE {\color{orange} AdamW: \;\;\;\;\;\; $\bnu_{t}\leftarrow \beta_{2}\bnu_{t-1}+ (1-\beta_{2})\bg_t^{\odot 2}$}
        
        \STATE {\color{blue} \nameA{}: \;\;\;\;\;\;\;\;$\bnu_{t}\leftarrow \beta_{2}\bom_{t-1}^{\odot 2}+ (1-\beta_{2})\bg_t^{\odot 2}$}
        \STATE \textbf{update model parameters}

        \STATE $\bw_{t}\leftarrow  (1-\eta_t\lambda)\bw_{t-1}-\eta_t \left(\frac{1}{\sqrt{\bnu_{t}}+\epsilon}  \odot  {\bom_t}\right)$
    \ENDWHILE
    
    \STATE \textbf{return} $\bw_t$
    \end{algorithmic}
    \label{alg:adamw}
\end{algorithm}

\subsection{The Properties of \nameA{}}
We next compare the behavior of \nameA{} and  that of AdamW. We note that it is very hard to analyze rigorously the update terms  for AdamW and \nameA{} because the correlations between the numerator and the denominator, also the correlations among historical gradients. The analysis here serves as a thought experiment with simplified assumptions (e.g., independence, distributional assumptions) to help illustrate conditions when the denominator of \nameA{} approximates that of AdamW. 

\textbf{Analytical comparison.}  The numerators of \nameA{} and AdamW are the same. To illustrate the behavior of denominators for a thought verification, we consider the following sequence $\{X_t\}$, where $X_t \sim \mathcal{N}(\mu, \sigma^2)$ are independent. Then the distribution of the exponentially weighted moving average (EMA) of their squared values 
\[
S_t = (1-\beta_2) X_t^2 + \beta_2 S_{t-1}, \quad t = 1,2,\dots,T.
\]
follows a weighted sum of noncentral chi-squared distributions. As $t$ becomes large, such a distribution tends to be a non-centered Gaussian distribution. We compute the mean and variance of $S_t$,
\begin{flalign*}
\mathbb{E}[S_t] &= (\mu^2+\sigma^2)(1-\beta_2^t),\\
\operatorname{Var}(S_t) &= \left(2\sigma^4+4\mu^2\sigma^2\right)
\frac{1-\beta_2}{1+\beta_2}\,(1-\beta_2^{2t}).
\end{flalign*}
Consequently, $ \mathbb{E}[S_\infty] = (\mu^2+\sigma^2)$, and $\operatorname{Var}(S_\infty) = (2\sigma^4 + 4\mu^2\sigma^2)(1-\beta_2)/(1+\beta_2)$.

On the other hand, the distribution of the exponential moving average of $X_t$, i.e. $M_t =  (1-\beta_1) X_t + \beta_1 M_{t-1}, \quad t = 1,2,\dots,$ follows a Gaussian distribution. 
 denominator of \nameA{} involves the following quantity, 
$V_t := \beta \,M_{t-1}^2 + (1-\beta)\,X_t^2$. 
Since \(X_t\) and \(M_{t-1}\) are independent, 
 \(V_t\) is the sum of two independent scaled noncentral chi–squared random variables with one degree of freedom. We have 

\begin{flalign*}
    \mathbb{E}[V_\infty] %&= \beta\,\mathbb{E}[M_{t-1}^2] + (1-\beta)\,\mathbb{E}[X_t^2]\\[1mm]
&=\mu^2+\sigma^2\left(1-\frac{2\beta\beta_1}{1+\beta_1}\right),\\
\operatorname{Var}(V_\infty)&=2\sigma^4\left[\beta^2\left(\frac{1-\beta_1}{1+\beta_1}\right)^2+(1-\beta)^2\right] \\
&+4\mu^2\sigma^2\left[\beta^2\frac{1-\beta_1}{1+\beta_1}+(1-\beta)^2\right].
\end{flalign*}

We note that if $\mu \gg \sigma$, which can be true when the gradient noise becomes considerably small as the batch size is extremely large. Alternatively, this suggest that the behavior of \nameA{} could be more resemble that of AdamW if the batch size get larger, fitting to practical setup in LLM pretraining.

By comparing $\mathbb{E}[S_\infty]$ and $\mathbb{E}[V_\infty]$, we note that if \(\mu \gg \sigma\), i.e., a regime achievable under  large batch sizes where gradient noise becomes negligible, \nameA{}’s behavior increasingly resembles that of AdamW. This alignment with AdamW’s dynamics under low-noise conditions mirrors practical LLM pretraining setups, where large batch sizes are standard.

Moreover, root operation is non-expansive. the denominators of \nameA{} and AdamW are quite close when $\mu \gg \sigma$, which could hold for very large batch size that is used in practice when training extremely large language models. We note that $\beta_2$ cannot be too close to 1.

For specific $\beta_1=0.9, \beta_2=0.95$, we have $\Var[S_\infty] \approx 0.0256(2\sigma^4+4\mu^2\sigma^2)$. The best $\beta=0.95$ to minimize the difference between the variance of $S_t$ and $V_t$, and $\Var[V_t] = 2\sigma^4*0.005 + 4\mu^2\sigma^2*0.05$.

\textbf{Empirical comparison between the update matrices of \nameA{} and AdamW.}  
We analyze the update matrices of AdamW and \nameA{} along the training trajectory of a GPT-2 Small model. The detailed experimental setup is provided in Section~\ref{subsec:gpt2-experiment}.  

To quantify the similarity between the updates, we compute the cosine similarity between the update matrices of \nameA{} and AdamW throughout the training process with AdamW. The results are presented in Figure~\ref{fig:cosine-similarity}. For comparison, we also include the cosine similarity between AdamW and the recently proposed Adam-mini \cite{zhang2024adam}. The results show that \nameA{} exhibits a strong alignment with AdamW, closely matching its update direction.

\begin{figure*}[tb]
\begin{center}
\begin{minipage}[t]{0.7\linewidth}
\centering
{\includegraphics[width=0.49\linewidth]{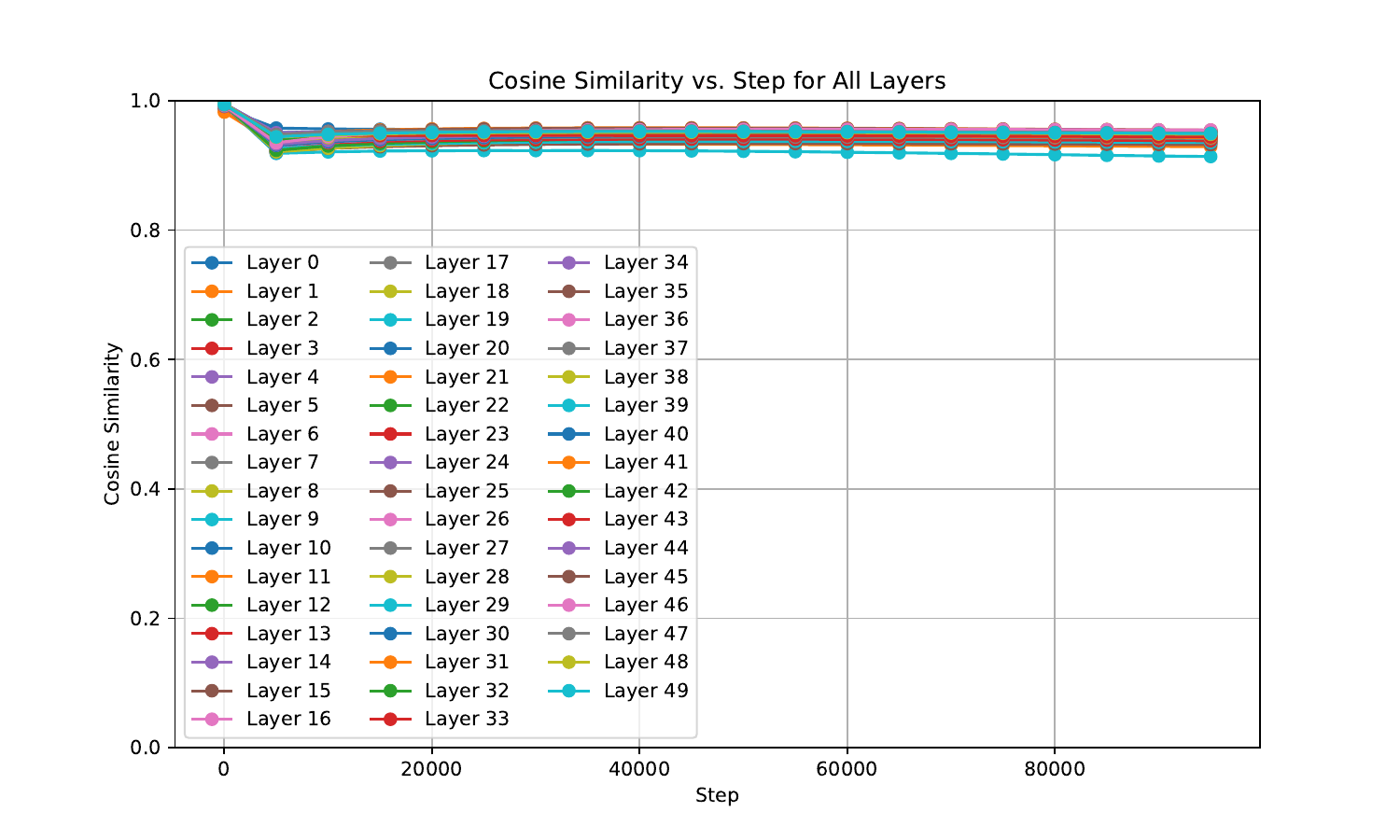}
\includegraphics[width=0.49\linewidth]{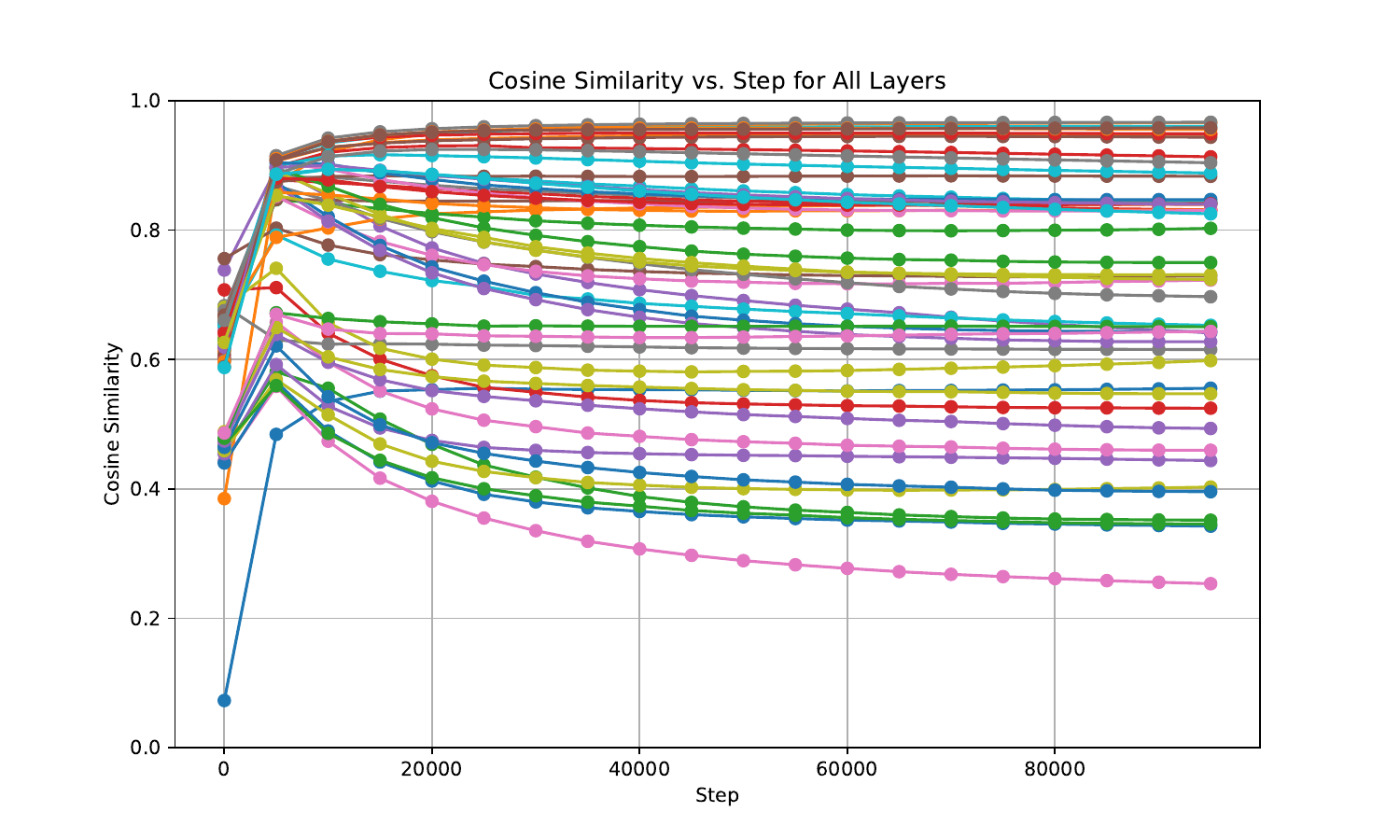}}
\caption{
The cosine similarities between the update matrices of \nameA{} and AdamW (\textbf{left}), Adam-mini and AdamW (\textbf{right}) for all layers of GPT2-Small model.  Across the training trajectory, the update direction of \nameA{}  closely aligns with that of AdamW.
}
\label{fig:cosine-similarity}
\end{minipage}
\hfill
\begin{minipage}[t]{0.27\linewidth}
\centering
{\includegraphics[width=\linewidth]{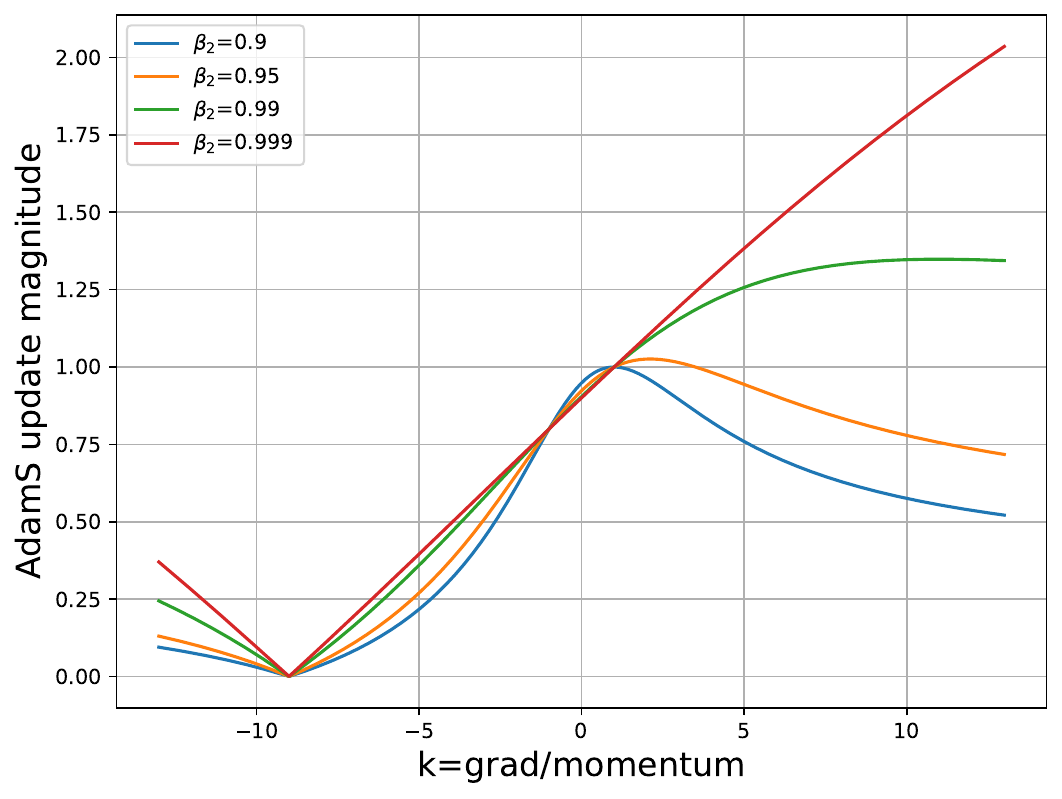}}
\caption{
The update magnitude of \nameA{} for grad/momentum varying with $\beta_1=0.9$ and $\beta_2=0.9,0.95,0.99,0.999$.
}
\label{fig:update-magnitude}
\end{minipage}

\end{center}
\end{figure*}

\textbf{The magnitude of \nameA{} update.} 
For \(\beta_1 = 0.9\), we plot the update magnitude of \nameA{} when the $gradient/momentum$ values span \([-13, 13]\), covering most values in practice,  in Figure~\ref{fig:update-magnitude}. We can see that overly large \(\beta_2\) values can destabilize updates by inflating the denominator’s sensitivity to outliers. To mitigate this, we recommend not setting \(\beta_2\) too large,  and a typical value of $\beta_2=0.95$ works well and aligns with empirical choice of AdamW for LLM pretraining.

\textbf{Memory cost and throughput of \nameA{}.} \nameA{} effectively reduces optimizer state memory usage by half. However, the extent of improvement in throughput and maximum batch size compared to the original AdamW depends on the model size and GPU type, as the primary bottleneck may be either memory or computation. Notably, as model size increases, the benefits of \nameA{} become more pronounced, aligning well with practical large language model (LLM) training scenarios. As shown in Table~\ref{tab:throughput}, \nameA{} can improve over AdamW in terms of throughput by almost 36\%, i.e., reducing the time 6.9s to 4.4s of passing a batch of tokens, for GPT2-XL pretraining.

\begin{table}[h]
    \centering
    \begin{tabular}{cccc}
        \hline
        Model & Optimizer & Max batch & Throughput\\
        \hline
        \multirow{2}{*}{774M} & AdamW & 10 & 2.0s\\
        & \nameA{} & 10 & 2.0s\\
        \hline
        \multirow{2}{*}{1.5B} & AdamW & 1 & 6.9s \\
        & \nameA{}& 3 & 4.4s \\
        \hline
    \end{tabular}
    \caption{Memory cost and throughput comparison between AdamW and \nameA{}. The maximum batch size (Max batch) is the largest allowable batch without Out of Memory and the throughput (Throughput) is measured by the time (in seconds) for one iteration of passing 480K tokens  with gradient accumulation. Experiment setup: 8 A100 GPUs with 40GB memory, training with GPT2-XL (1.5B) and GPT2-Large (774M).}%\footnote{The model names and parameter numbers follow the GPT-2 descriptions by OpenAI in HuggingFace.}
    \label{tab:throughput}

\end{table}

\section{Convergence of \nameA{}}

This section establishes the theoretical convergence of \nameA{}. We first introduce another key assumption on the gradient noises.

\begin{assumption}[Sub-gaussian noise]
\label{assum: noise}
We assume that the stochastic noise $\bg_t$ is unbiased, i.e.,  $\mathbb{E}^{|\mathcal{F}_t} \bg_t=\bG_t$. We further assume $\bg_t$ is centered with sub-gaussian norm, i.e., there exists some positive constant $R $, such that  $\mathbb{P}^{|\gF_t} (\Vert\bg_t - \nabla f(\bw_t) \Vert \ge s) \le  2e^{-\frac{s^2}{2R^2}}$.
\end{assumption}

Assumption \ref{assum: noise} is one of the weakest assumptions on the noise in existing literature, and generalizes bounded gradient assumption \cite{defossezsimple} and bounded noise assumption \cite{li2023convergence}. Based on Assumption \ref{assum: objective} and \ref{assum: noise}

\begin{theorem}
\label{thm: parameter_agnostic}
     Let Assumptions \ref{assum: objective} and \ref{assum: noise} hold. Then, setting $\eta_t = \tilde{\mathcal{O}}(\frac{1}{\sqrt{T}})$, $\bone = 1- \tilde{\Theta}(\frac{1}{\sqrt{T}})$, and $\btwo = 1 - \tilde{\Theta}(\frac{1}{T})$ with $\frac{1-\bone}{\eta} \ge C$, where $C$ is some constant defined in Eq. (\ref{eq: def C}) , we have that \nameA{} in Algorithm \ref{alg:adams} satisfies
    \begin{equation*}
        \E \min_{t\in [1,T]}\Vert \nabla f(\bw_t) \Vert \le \tilde{\mathcal{O}} \left(\frac{1}{\sqrt[4]{T}}\right).
    \end{equation*}
\end{theorem}

\begin{proof}
The proof is relegated to Appendix~\ref{sec: appendix proof} due to space constraints.  
\end{proof}

The derived convergence rate matches the known lower bound of $\Omega(1/\sqrt[4]{T})$  for any gradient-based optimizer, including AdamW \citep{arjevani2022lower}. This result not only demonstrates that the convergence rate in Theorem \ref{thm: parameter_agnostic} is tight —-achieving the theoretically optimal bound —- but also provides a rigorous theoretical guarantee for AdamS’s efficiency in optimizing Transformer architectures.

\section{Empirical Performance of \nameA{}}
In this section, we apply \nameA{} for large language model pretraining tasks and post-training tasks to demonstrate that \nameA{} can achieve performance comparable to AdamW with similar hyperparameters while requiring significantly less memory.

\begin{figure*}[tb]
    \centering
    \begin{minipage}{0.3\linewidth}
        \centering
        \includegraphics[width=\linewidth]{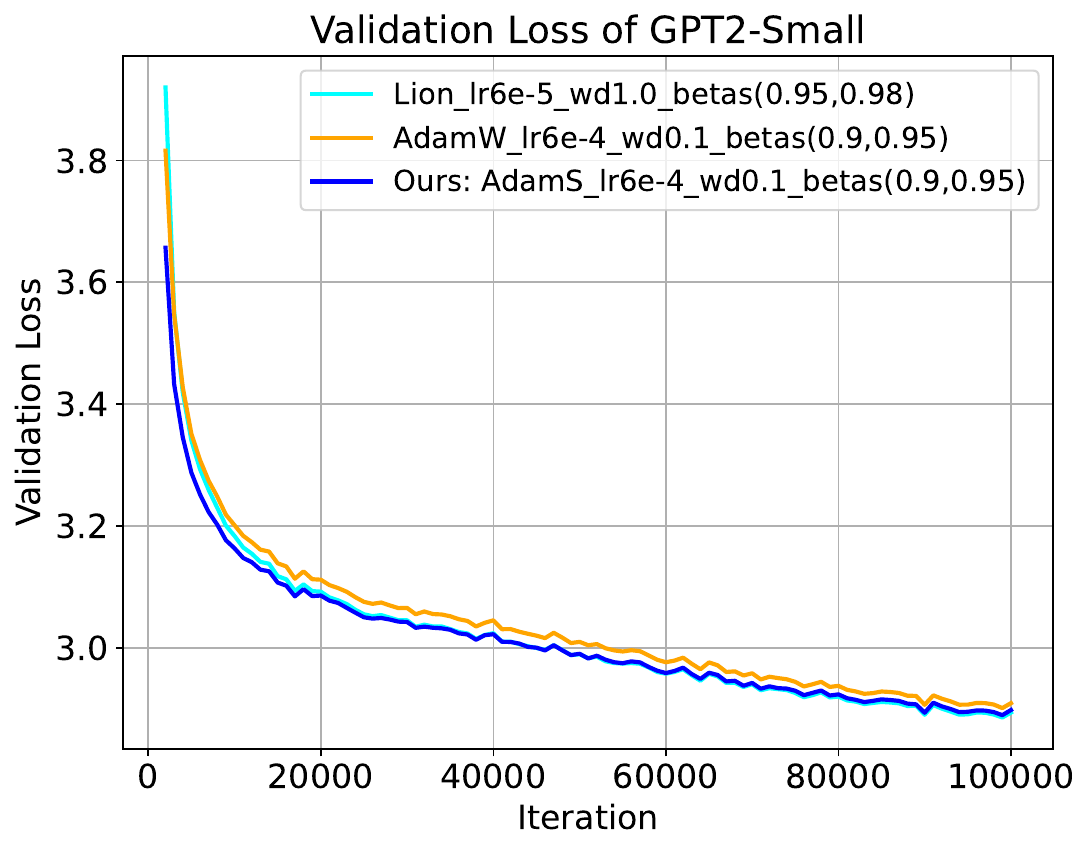}
        %\subcaption{First subfigure caption}\label{fig:sub1}
    \end{minipage} \hfill
    \begin{minipage}{0.3\linewidth}
        \centering
        \includegraphics[width=\linewidth]{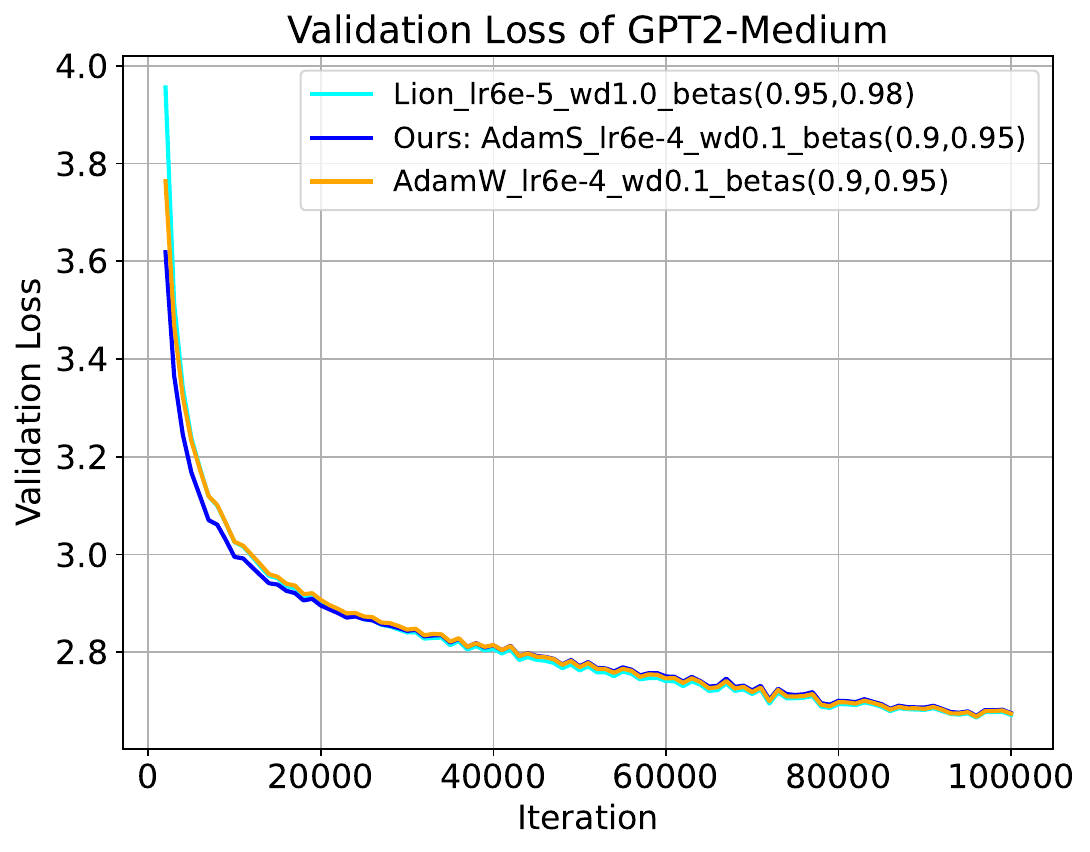}
        %\subcaption{Second subfigure caption}\label{fig:sub2}
    \end{minipage} \hfill
    \begin{minipage}{0.3\linewidth}
        \centering
        \includegraphics[width=\linewidth]{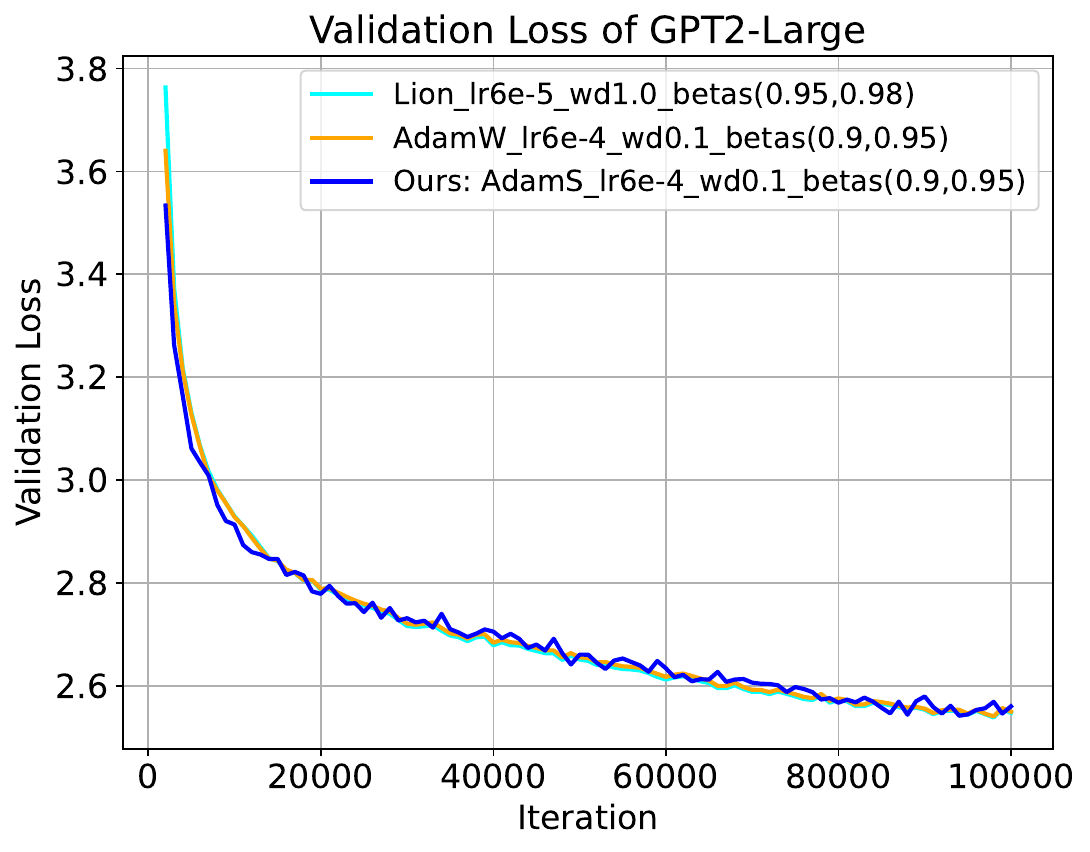}
        %\subcaption{Third subfigure caption}\label{fig:sub3}
    \end{minipage}
    \caption{
    Validation loss curves for pretraining GPT-2 models. Across three different model sizes and with the same hyperparameters as AdamW, the proposed \nameA{} achieves convergence comparable to baseline methods—without the need to store AdamW’s second-moment estimates.}
    \label{fig:gpt2-curves}
\end{figure*}

\subsection{GPT2 experiments}\label{subsec:gpt2-experiment}

In this experiment, we demonstrate that \nameA{} achieves performance comparable to AdamW  for pretraining GPT2 \cite{gpt2} on the OpenWebText dataset \cite{pile} using the popular nanoGPT codebase\footnote{https://github.com/karpathy/nanoGPT}. We evaluate three variants: GPT2 Small (125M parameters), GPT2 Medium (355M parameters), and GPT2 Large (770M parameters).

\textbf{Baselines.}  We primarily compare \nameA{} with AdamW~\citep{loshchilov2019adamw}, the most widely used optimizer in language modeling tasks, and Lion~\cite{chen2023symbolic}, a recently proposed optimizer that eliminates the need for second-moment  estimates, discovered by symbolic search.

We adopt typical hyperparameter choices, following the settings used in \cite{zhang2024adam,liu2023sophia}. For AdamW, we set \((\beta_1,\beta_2)=(0.9,0.95)\) with a weight decay of 0.1, and we use a learning rate of \(6\times10^{-4}\) for  the GPT2 Small model and $\text{lr}=3\times10^{-4}$ for the GPT2 Medium and GPT2 Large models.   For Lion, as suggested by \citet{chen2023symbolic}, we use \((\beta_1,\beta_2)=(0.95,0.98)\), set the learning rate to \(0.1 \times \text{lr}_{\text{AdamW}}\), and choose a weight decay of \(10 \times \text{weight\_decay}_{\text{AdamW}}\). For \nameA{}, we use the same hyperparameters as AdamW; that is, \(\text{lr}=\text{lr}_{\text{AdamW}}\), \((\beta_1,\beta_2)=(0.9,0.95)\), and \(\text{weight\_decay}=\text{weight\_decay}_{\text{AdamW}}\).

\textbf{Implementation.} Following standard practices, for all GPT-2 models, we set the context length to be 1024 tokens. We use a batch size of 480  and employ a cosine learning rate schedule, setting the final learning rate to $0.1 \times \text{lr}$ as suggested by \citet{rae2021scaling}. We employ gradient clipping by norm with a threshold of 1.0, and we use a fixed warm-up period of 2,000 steps. The algorithms are implemented in PyTorch~\citep{paszke2019pytorch}, and training is conducted in float16 precision on clusters equipped with Nvidia Ampere or Hopper GPUs for the GPT2-Small, Medium, and Large models.

\textbf{Results.} The results are shown in Figure~\ref{fig:gpt2-curves}. As observed in Figure~\ref{fig:gpt2-curves}, the performance of \nameA{} closely mirrors the AdamW curves across all three model sizes throughout the training process. This is achieved using the same hyperparameters as those for AdamW. Further details and longer training steps are provided in Appendix~\ref{app:more-experiments}.

\subsection{Llama2 Pretraining Experiments}

In this experiment, we confirm the behavior of AdamS for 
pretraining an even larger model Llama2-7B~\citep{touvron2023llama2}. It is trained with  the well-known Torchtitan library\footnote{https://github.com/pytorch/torchtitan} on the C4 dataset~\citep{raffel2020exploring}.

\textbf{Training setup.} We use the same hyperparameters for Llama2-7B pretraining as those in \citet{touvron2023llama2}. The training setup involves a batch size of 1024, a context length of 4096, where the batch size is 4M in terms of tokens, and gradient clipping with a maximum norm of 1.0. The learning rate schedule includes a fixed 2000 step warmup followed by linear decay. The training is conducted in bfloat16 precision on one node equipped with 8 Nvidia  Hopper GPUs with 80G memory.  Due to budget limitations, we train the model for 8K steps, which corresponds to processing over 32B tokens. The validation loss is evaluated every 200 steps.

\textbf{Other hyperparameter choice.} For AdamW, we use \((\beta_1,\beta_2) = (0.9, 0.95)\), a peak learning rate of \(3\times10^{-4}\), and a weight decay of 0.1. For \nameA{}, Adam-mini, we use the same hyperparameters as AdamW. For Lion, we use the recommended settings: \(\text{lr} = 0.1 \times \text{lr}_{\text{AdamW}}\) and \(\text{weight\_decay} = 10 \times \text{weight\_decay}_{\text{AdamW}}\).

\textbf{Results.} The results are summarized in Figure~\ref{fig:llama}. As shown in Figure~\ref{fig:llama}, \nameA{} achieves slightly better convergence than other strong baselines: AdamW, Adam-mini and Lion across the training trajectory under the same  default hyperparameters as in \citet{touvron2023llama2}. Notably, training with \nameA{} reduces memory consumption by 20\% when using a popular training recipe, i.e., Fully Sharded Data Parallel (FSDP) technique \cite{paszke2019pytorch} on 4 NVIDIA Hopper GPUs. By eliminating the need to communicate second-moment estimates across GPUs and nodes, \nameA{} alleviates communication bottlenecks, a critical advantage for low-end GPU clusters where inter-card bandwidth is often a limiting factor.

Due to space limit, we present a setting of Llama2-13B pretraining with smaller batch size in Appendix~\ref{app:more-experiments}.

\subsection{ RL Post-training of LLMs}

\begin{figure}[htb!]
    \centering
    \begin{minipage}{0.49\linewidth}
        \centering
        \includegraphics[width=\linewidth]{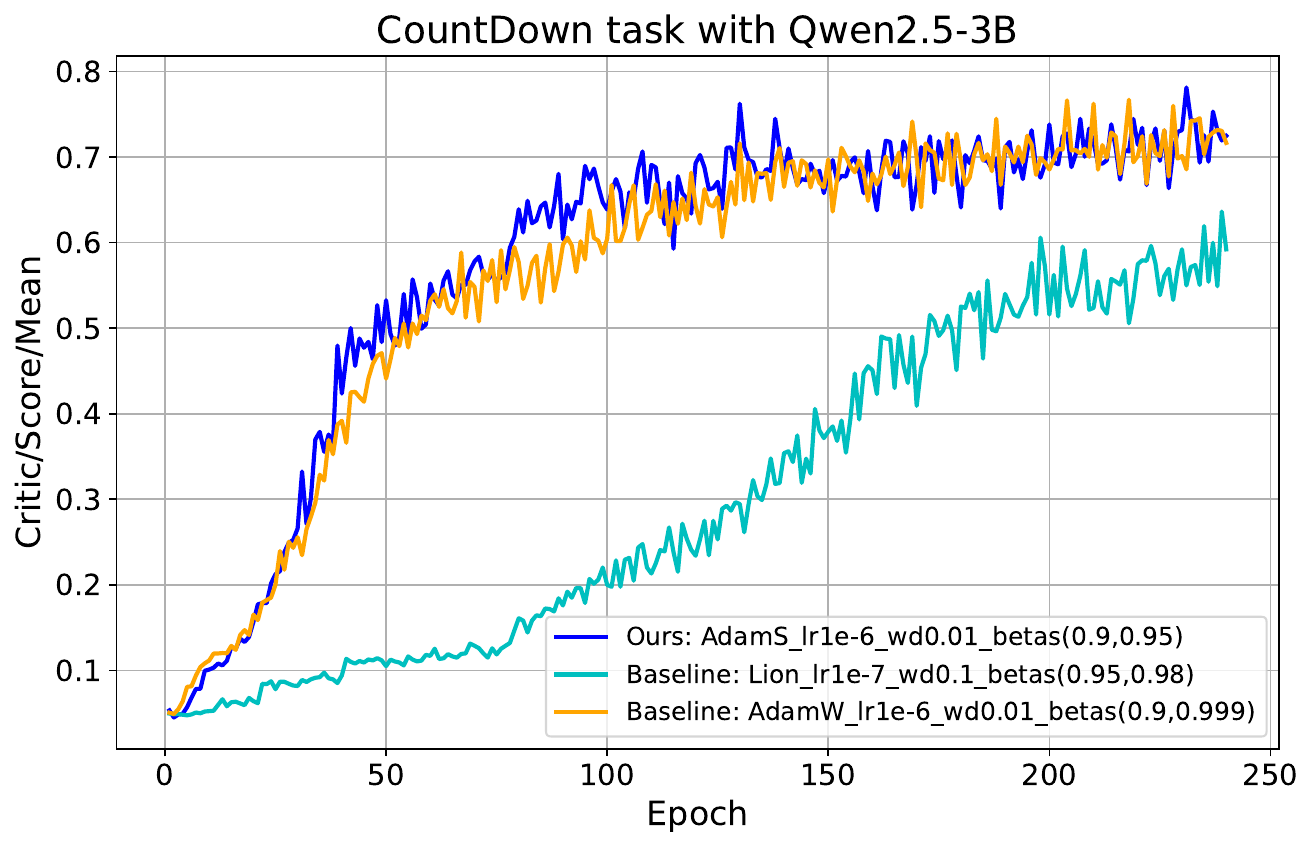}
        %\subcaption{First subfigure caption}\label{fig:sub1}
    \end{minipage} \hfill
    \begin{minipage}{0.49\linewidth}
        \centering
        \includegraphics[width=\linewidth]{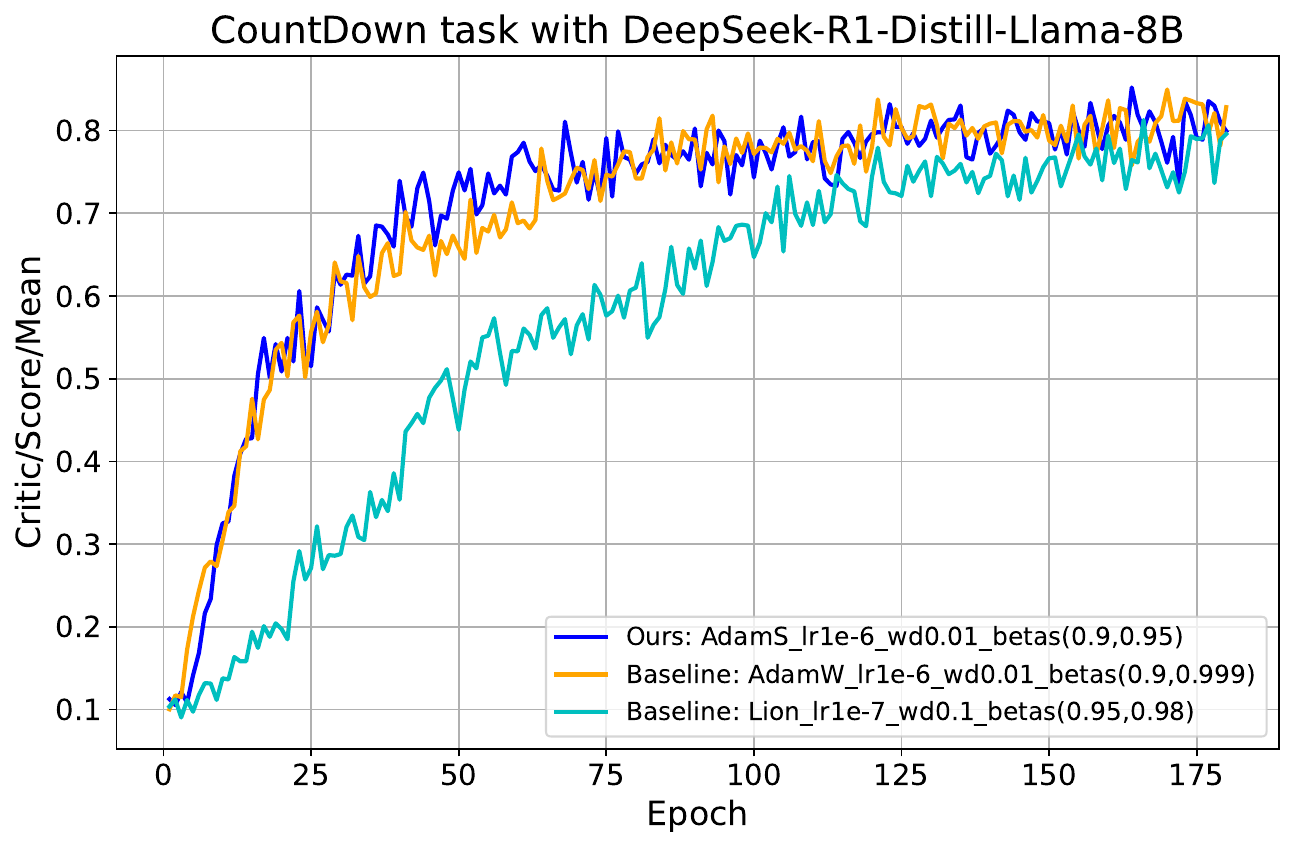}
        %\subcaption{Second subfigure caption}\label{fig:sub2}
    \end{minipage} \hfill
    \caption{Mean critic scores for reinforcement learning (RL) post-training using the GRPO algorithm on the CountDown task are presented for the Qwen2.5-3B and DeepSeek-R1-Distill-Llama-8B models. The proposed \nameA{} closely resembles AdamW’s performance trajectory, achieving similar convergence curves. In contrast, Lion with default hyperparameters demonstrates significantly slower convergence under the same conditions.}
    \label{fig:r1-zero}
    \vspace{-2mm}
\end{figure}

In this experiment, we leverage the TinyZero project~\cite{tinyzero} that  provides a clean, minimal, and accessible reproduction of the DeepSeek R1-Zero framework~\cite{guo2025deepseek}. We choose two models Qwen2.5-3B~\cite{qwen2.5} and R1-Distilled-Llama8B~\cite{guo2025deepseek} and evaluate the DeepSeek R1-Zero Group Relative Policy Optimization (GRPO) method on the Countdown Numbers Game. In this task, the model is asked to use a set of randomly chosen numbers along with basic arithmetic operations (\(+, -, \times, \div\)) to reach a target number, with each number used only once.

\textbf{Hyperparameter choice.} For the baseline AdamW setup, we use the default learning rate of \(1 \times 10^{-6}\), \((\beta_1,\beta_2) = (0.9, 0.999)\), and a weight decay of \(1 \times 10^{-2}\). We test the Group Relative Policy Optimization (GRPO) reinforcement learning algorithm \cite{shao2024deepseekmath,guo2025deepseek} with all other hyperparameters maintained as in the original project. For \nameA{}, we adopt the same hyperparameters as AdamW, except that we set \(\beta_2 = 0.95\) for good stability, as explained in Section 2.2 and Figure~\ref{fig:update-magnitude}. For Lion, we follow the recommendations from the original paper by setting \(\text{lr} = 0.1 \times \text{lr}_{\text{AdamW}}\), \(\text{weight\_decay} = 10 \times \text{weight\_decay}_{\text{AdamW}}\), and \((\beta_1, \beta_2) = (0.95, 0.98)\).

\textbf{Implementation.} The TinyZero framework implements the DeepSeek R1-Zero reinforcement learning objective, which encourages the models to generate an extended chain-of-thought before producing a final answer. This approach aims to guide the models in developing a structured reasoning process for the Countdown Numbers Game.

\textbf{Results.}  The results are shown in Figure~\ref{fig:r1-zero}.  
Across two distinct base models—Qwen2.5-3B and the distilled DeepSeek-R1-Distill-Llama-8B—the score curves of \nameA{} closely align with those of AdamW, even occasionally surpassing its validation performance. This consistency underscores \nameA{}’s ease of adoption across diverse tasks, requiring no specialized tuning. In contrast, Lion, when applied with its default hyperparameters, exhibits much slower convergence under identical experimental conditions.

This point holds significant practical value: while many optimizers excel in some specific scenarios with carefully tuned hyperparameters, \nameA{}’s robust performance easily generalizes to unseen tasks without much hyperparameter tuning, making it a scalable solution for both current and future applications.

\subsection{Sensitivity to Hyperparameters}

We ablate the hyperparameter choices of $(\beta_1,\beta_2)$ of \nameA{}. Table~\ref{tab:beta-results} shows the performance sensitivity to $(\beta_1, \beta_2)$ for the GPT2-small pretraining task. The numbers are validation loss after training 100K iterations with other hyperparameters the same as those in Section~\ref{subsec:gpt2-experiment}.
\begin{table}[htbp]
  \centering
\setlength{\tabcolsep}{4pt}          % adjust if needed
  
  \begin{tabular}{cccccc}
    \toprule
    $\beta_1 \backslash \beta_2$   & 0.90   & 0.95   & 0.98   & 0.99  &0.999 \\
    \midrule
    0.90                          & 2.902  & 2.898  & 2.904  & 2.904 &2.902  \\
    0.95                           &  -      & 2.897  & 2.892  & 2.898 & 3.460 \\
    \bottomrule
  \end{tabular}
  \caption{Validation loss for different $(\beta_1,\beta_2)$ pairs of GPT2-small pretraining with \nameA{}.}
  \label{tab:beta-results}
\end{table}

These results indicate that AdamS is robust and stable  over a wide range of configurations except for very large $(\beta_1,\beta_2)$ pair, supporting its practical use and easy adoption.

\section{Discussion and Conclusion} \label{sec:discussion}

We have proposed a well-motivated design of LLM optimizer, \nameA{}, which can serve as the newly default optimizer for training large-scale language model training, because of its efficiency, simplicity, and theoretical rigor. By replacing second-moment estimation with a momentum-weighted root mean square denominator, the method achieves computational parity with SGD while matching the performance of Adam-family optimizers in both pretraining and post-training scenarios. Its seamless integration into existing frameworks—enabled by AdamW-compatible hyperparameters and model-agnostic design—removes adoption barriers, offering practitioners a "plug-and-play" upgrade. 

The theoretical property of \nameA{} has also been extensively analyzed, including the update magnitude estimation and convergence under relaxed smoothness assumption. This theoretical insight, coupled with empirical validation across architectures (e.g., GPT-2, Llama2) and training paradigms (e.g., RL post-training), demonstrates robustness to scale and task diversity. Notably, \nameA{}’s elimination of communication overhead for second-moment statistics positions it as a scalable solution for communication-bounded environments.

Future work may explore \nameA{}’s applicability to emerging architectures and its synergies with advanced parallelism strategies  for next-generation LLM development.

\section*{Limitations}

While \nameA{} achieves promising performance across tasks and model scales, several limitations deserve discussion. First, our experiments were constrained by computational resources, particularly in pretraining scenarios (e.g., Llama2-7B \& 13B). Validating \nameA{}’s efficacy at extreme scales—such as models beyond 100B parameters, datasets exceeding 1T tokens, or emerging architectures like Mixture of Experts (MoE)—remains critical for confirming its scalability in production-grade pipelines. Such studies would require computational resources far beyond our current capacity.

Second, fairly benchmarking optimizers has inherent challenges due to confounding variables like learning rate schedules, weight decay policies, optimizer-specific hyperparameters (e.g., \nameA{}’s momentum weighting), and implementation efficiency. While our work compares \nameA{} against strong baselines (AdamW, Lion) using established hyperparameters, we limited exhaustive hyperparameter searches across all optimizers to maintain parity.

These limitations underscore the need for community-driven standardization of optimizer evaluations and deeper exploration of \nameA{}’s behavior in extreme-scale regimes. To foster reproducibility, we will release all code, configurations, and training protocols to facilitate reproducibility and encourage broader investigation.

\newpage
\appendix
\onecolumn

\section{Algorithms: Lion and Adam-mini}\label{app:lion}

Two related algorithms used as baselines in the paper are presented as follows.
\begin{algorithm}
    % \caption{{\color{orange} Lion } v.s. {\color{blue} SignSGDM}}
    \caption{{ Lion Optimizer \cite{chen2023symbolic}}}
    \begin{algorithmic}[1]
    \STATE \textbf{Input:} momentum parameters $\beta_1$, $\beta_2$, weight decay $\lambda$, learning rate $\eta$, objective function $f$
    \STATE \textbf{Initialize} starting point $\bw_0$, initial  $\bom_0\leftarrow 0, t\leftarrow 0$
    \WHILE{$\bw_t$ not converged}
        \STATE $t\leftarrow t+1$
        \STATE $\bg_t \leftarrow \nabla_{\bw}{f(\bw_{t-1})}$
        \STATE \#\#\# {update model parameters}
        %\STATE 
        % {\color{blue} SignSGDM: $\bu_t \leftarrow  \bom_{t}$}
        \STATE {$\bu_t \leftarrow \beta_1 \bom_{t-1} + (1-\beta_1)\bg_t$}
        \STATE $\bw_t \leftarrow \bw_{t-1} - \eta_t(\text{sign}(\bu_t) + \lambda\bw_{t-1})$
        \STATE \#\#\# {update momentum tracking}
        \STATE $\bom_t \leftarrow \beta_2 \bom_{t-1} + (1 - \beta_2)\bg_t$
    \ENDWHILE
    \STATE \textbf{return} $\bw_t$
    \end{algorithmic}
    \label{alg:lion}
\end{algorithm}

\begin{algorithm}
\caption{{Adam-mini}~\citep{zhang2024adam}}\label{alg:adam-mini}
\begin{algorithmic}[1]
\STATE \textbf{Input:} weight-decay coefficient $\lambda$ and current step $t$
\STATE \textbf{Partition:} Partition params into \texttt{param\_blocks} by \textbf{Principle}~\ref{principle:adam-mini}
\FOR{\texttt{param in param\_blocks}}
    \STATE $g = \texttt{param.grad}$
    \STATE $\texttt{param} = \texttt{param} - \eta_t \cdot \lambda \cdot \texttt{param}$
    \STATE $m = (1 - \beta_1) \cdot g + \beta_1 \cdot m$
    \STATE $\hat{m} = \frac{m}{1 - \beta_1^t}$
    \STATE {\color{blue}$v = (1 - \beta_2) \cdot \text{mean}(g \odot g) + \beta_2 \cdot v$}
    \STATE $\hat{v} = \frac{v}{1 - \beta_2^t}$
    \STATE $\texttt{param} = \texttt{param} - \eta_t \cdot \frac{\hat{m}}{\sqrt{\hat{v}} + \epsilon}$
\ENDFOR
\end{algorithmic}
\end{algorithm}
\begin{principle}[\citet{zhang2024adam} Principle 1] \label{principle:adam-mini}
 We should partition parameters into blocks, such that each parameter block is associated
with the smallest dense sub-block in Hessian.
\end{principle}

It is worthy noting that Algorithm~\ref{alg:adam-mini} requires  partition of parameters based on the Hessian structure of the architecture, which makes it not able to be model agnostic.
% \section{Sensitivity to Hyperparameters}

% We add here a detailed table showing performance sensitivity to key hyperparameters $(\beta_1, \beta_2)$, learning rate) for the GPT2-small pretraining task under 100K iterations. The numbers are validation loss after training 100K iterations with other hyperparameters the same as those in the paper.

% \begin{table}[htbp]
%   \centering
% \setlength{\tabcolsep}{4pt}          % adjust if needed
%   \caption{Perplexity for different $(\beta_2,\beta_1)$ pairs.}
%   \label{tab:beta-results}
%   \begin{tabular}{lccccc}
%     \toprule
%     $\beta_1 \backslash \beta_2$ & 0.80   & 0.90   & 0.95   & 0.98   & 0.99   \\
%     \midrule
%     0.90                         & 2.907  & 2.902  & 2.898  & 2.904  & 2.904  \\
%     0.95                         &        &        & 2.897  & 2.892  & 2.898  \\
%     \bottomrule
%   \end{tabular}
% \end{table}

% The analysis indicates that AdamS is robust, with stable behavior over a wide range of configurations, supporting its practical use without additional tuning.

\section{More Experiments}\label{app:more-experiments}
We put more experiments here due to space limit.
\subsection{Llama2-13B Pretraining Experiments}

In this experiment, we confirm the behavior of AdamS for 
pretraining an even larger model Llama2-13B~\citep{touvron2023llama2}. It is trained with  the well-known Torchtitan library\footnote{https://github.com/pytorch/torchtitan} on the C4 dataset~\citep{raffel2020exploring}.

\textbf{Training setup.}   The training setup involves a batch size of \(2 \times 8\), a context length of 2048, and gradient clipping with a maximum norm of 1.0. The learning rate schedule includes a fixed 100-step warmup followed by linear decay. The training is conducted in bfloat16 precision on one node equipped with 8 Nvidia  Hopper GPUs with 80G memory.  Due to budget limitations, we train the model for 30K steps, which corresponds to processing over 0.96B tokens. This follows the setting used in Adam-mini~\citep{zhang2024adam}.

\textbf{Other hyperparameter choice.} For AdamW, we use \((\beta_1,\beta_2) = (0.9, 0.95)\), a peak learning rate of \(1\times10^{-4}\), and a weight decay of 0.1. For \nameA{}, we use the same hyperparameters as AdamW. %For Lion, we use the recommended settings: \(\text{lr} = 0.1 \times \text{lr}_{\text{AdamW}}\) and \(\text{weight\_decay} = 10 \times \text{weight\_decay}_{\text{AdamW}}\). 

% \textbf{Implementation.} The training setup involves a batch size of \(2 \times 8\), a context length of 2048, and gradient clipping with a maximum norm of 1.0. The learning rate schedule includes a fixed 100-step warmup followed by linear decay. The training is conducted in bfloat16 precision on one node equipped with 8 Nvidia  Hopper GPUs with 80G memory.  Due to budget limitations, we train the model for 30K steps, which corresponds to processing over 0.96B tokens.

\textbf{Results.} The results are summarized in Figure~\ref{fig:llama2-13B}. As shown in Figure~\ref{fig:llama2-13B}, \nameA{} achieves performance nearly identical to AdamW across the training trajectory under the same  hyperparameters. %Notably, training with \nameA{} reduces memory consumption by 20\% when using a popular training recipe, i.e., Fully Sharded Data Parallel (FSDP) technique \cite{paszke2019pytorch} on 4 NVIDIA Hopper GPUs. Additionally, by eliminating the need to communicate second-moment estimates across GPUs and nodes, \nameA{} alleviates communication bottlenecks, a critical advantage for low-end GPU clusters where inter-card bandwidth is often a limiting factor.

\begin{figure}[htb]
\begin{center}
\begin{minipage}[t]{0.49\linewidth}
\centering
{\includegraphics[width=\linewidth]{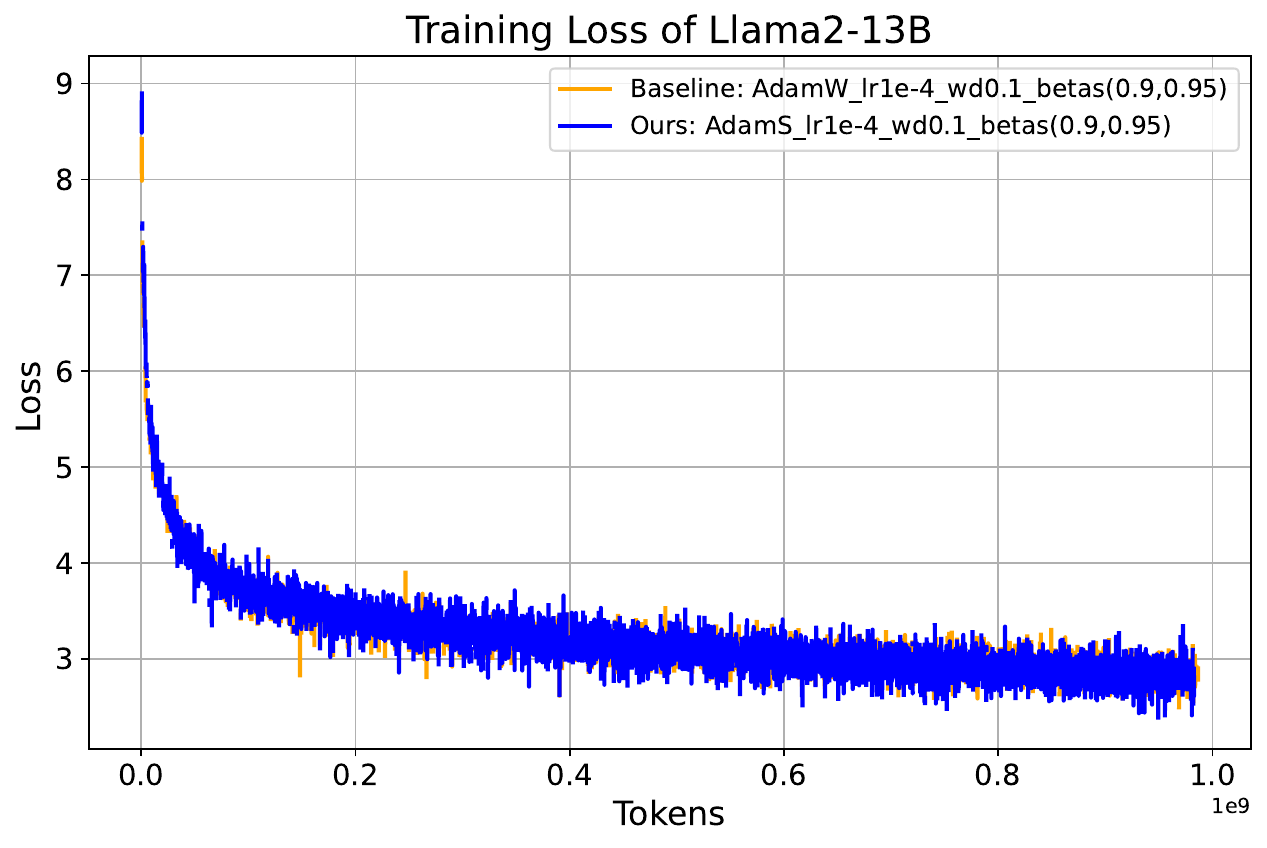}} 
 % (a) GPT2-small (124M)
\end{minipage}

\end{center}
\caption{
Training and validation loss curves for pretraining LLaMA 2–13B models. The proposed \nameA{} achieves convergence comparable to or better than baseline methods under the same hyperparameter settings as AdamW, while eliminating the need to store AdamW’s second-moment estimates.
}
\label{fig:llama2-13B}
\end{figure}

\subsection{GPT2 Experiments}

\textbf{Longer pretraining.} 
In this part, the pretraining setup is the same as Section~\ref{subsec:gpt2-experiment}, we present the final validation losses after pretraining for 100K and 300K in Table~\ref{tab:gpt2-adams}. We can see that the performance of \nameA{} closely mirrors the AdamW curves across all three model sizes throughout the training process. This is achieved using the same hyperparameters as those for AdamW. %Additionally, we confirm this behavior over 300K training iterations, which corresponds to processing over 14.4 billion tokens. Due to resource constraints, we were unable to train larger models or process more tokens.

\begin{table*}[htb!]
    \centering
    
    % \resizebox{1.\linewidth}{!}{
    \begin{tabular}{lcccccc}
    \toprule
      Model  &Iterations &Optimizer  &Peak LR & Weight decay&   $(\beta_1,\beta_2)$     &Valid. PPL  \\ \midrule
    124M & 100K & AdamW  &6e-4& 0.1 & $( 0.9, 0.95)$ & 2.902 \\
    &&Lion  & 6e-5& 1.0 & $(0.95, 0.98)$  & 2.886 \\
    &&\nameA{}  & 6e-4& 0.1 & $(0.9, 0.95)$  & 2.890\\ 
    \cmidrule{2-7}
    &300K & AdamW  & 6e-4& 0.1 &$(0.9, 0.95)$  & 2.867 \\
    &&Lion  & 6e-5&1.0 &$(0.95, 0.98)$  & 2.847 \\
    &&\nameA{}  & 6e-4& 0.1 & $(0.9, 0.95)$  & 2.866 \\ 
    
    % \midrule
   
    %   350M & 100K & AdamW  &3e-4& 0.1 & $( 0.9, 0.95)$ & 2.668 \\
    % &&Lion  & 3e-5& 1.0 & $(0.95, 0.98)$  & 2.667 \\
    % &&\nameA{}  & 3e-4& 0.1 & $(0.9, 0.95)$  & 2.669  \\ \cmidrule{2-7}
    % &300K & AdamW  & 3e-4& 0.1 &$(0.9, 0.95)$  & 2.630 \\
    % &&Lion  & 3e-5& 1.0 &$(0.95, 0.98)$  & 2.606 \\
    % &&\nameA{}  & 3e-4& 0.1 & $(0.9, 0.95)$  & 2.636 \\ 

    \bottomrule
    \end{tabular}%}
    \caption{Comparison of Lion, AdamW and \nameA{} on training GPT2 with the OpenWebText dataset.}\label{tab:gpt2-adams}
\end{table*}

\textbf{Comparison with other optimizers.} As the Adafactor and SM3 performs strictly inferior to AdamW for GPT2-small pretraining, as shown in Figure 8 of Zhang et al. 2024 (the Adam-mini paper) and we omit the comparison here.

We add experiments on GPT2-small pretraining with Adagrad and RMSProp. We note that there are not public training recipes for RMSprop and Adagrad of large language model pretraining. We use the same learning rate and learning rate decay schedule as those of AdamW, and use other hyperparameters as default. The results are shown in below.

\begin{table}[h]
\centering

\begin{tabular}{lcccc}
\toprule
\textbf{Metric} & \textbf{AdamW} & \textbf{Adagrad} & \textbf{RMSprop} & \textbf{AdamS} \\
\midrule
Valid loss of GPT-2 small & 2.909 & 3.887 & 3.089 & 2.898 \\
\bottomrule
\end{tabular}
\caption{Validation loss of GPT-2 small after 100K training iterations using different optimizers.}
\label{tab:gpt2_valid_loss_transposed}
\end{table}

\section{Derivation of the Learning Rate under $(L_0,L_1)$ Smoothness} \label{app:descent}

The \emph{smoothness constant \( L(\bw) \)} governs how much the gradient can change locally. If \( L(\bw) \) scales with \( \|\nabla f(\bw)\| \), the curvature (and thus the risk of overshooting) increases with the gradient's magnitude. This necessitates a smaller learning rate when the gradient is large and allows a larger rate when the gradient is small.

Here is a brief derivation for the above intuition.

\emph{Descent Lemma}: For \( L(\bw) \)-smooth \( f \), the update \( \bw_{t+1} = \bw_t - \eta \nabla f(\bw_t) \) satisfies:
   \begin{flalign*}
   f(\bw_{t+1}) \leq& f(\bw_t) - \eta \|\nabla f(\bw_t)\|^2 + \frac{\eta^2 L(\bw_t)}{2} \|\nabla f(\bw_t)\|^2.       
   \end{flalign*}

Substitute \( L(\bw_t) \leq  L_0+L_1\|\nabla f(\bw_t)\| \):
   \[
   f(\bw_{t+1}) \leq f(\bw_t) - \eta \|\nabla f(\bw_t)\|^2 + \frac{\eta^2 (L_0+L_1\|\nabla f(\bw_t)\|)}{2} \|\nabla f(\bw_t)\|^2.
   \]
   % Simplifying:
   % \[
   % f(\bw_{t+1}) \leq f(\bw_t) - \eta \|\nabla f(\bw_t)\|^2 + \frac{L_1 \eta^2}{2} \|\nabla f(\bw_t)\|^3.
   % \]
Ensure Decrease: For \( f(\bw_{t+1}) \leq f(\bw_t) \), require:
   \[
   -\eta \|\nabla f(\bw_t)\|^2 + \frac{L_0+L_1\|\nabla f(\bw_t)\| }{2} \eta^2\|\nabla f(\bw_t)\|^2 \leq 0.
   \]
   Factor out \( \eta \|\nabla f(\bw_t)\|^2 \):
   \[
   \eta \|\nabla f(\bw_t)\|^2 \left( -1 + \eta\frac{L_0+L_1\|\nabla f(\bw_t)\|}{2} \right) \leq 0.
   \]
   This implies:
   \[
   \eta \leq \frac{2}{L_0+L_1\|\nabla f(\bw_t)\|}.
   \]
   
% 2. Substitute \( L(\bw_t) \leq  L_0+L_1\|\nabla f(\bw_t)\| \):
%    \[
%    f(\bw_{t+1}) \leq f(\bw_t) - \eta \|\nabla f(\bw_t)\|^2 + \frac{\eta^2 C \|\nabla f(\bw_t)\|}{2} \|\nabla f(\bw_t)\|^2.
%    \]
%    Simplifying:
%    \[
%    f(\bw_{t+1}) \leq f(\bw_t) - \eta \|\nabla f(\bw_t)\|^2 + \frac{C \eta^2}{2} \|\nabla f(\bw_t)\|^3.
%    \]
% 3. Ensure Decrease: For \( f(\bw_{t+1}) \leq f(\bw_t) \), require:
%    \[
%    -\eta \|\nabla f(\bw_t)\|^2 + \frac{C \eta^2}{2} \|\nabla f(\bw_t)\|^3 \leq 0.
%    \]
%    Factor out \( \eta \|\nabla f(\bw_t)\|^2 \):
%    \[
%    \eta \|\nabla f(\bw_t)\|^2 \left( -1 + \frac{C \eta}{2} \|\nabla f(\bw_t)\| \right) \leq 0.
%    \]
%    This implies:
%    \[
%    \eta \leq \frac{2}{C \|\nabla f(\bw_t)\|}.
%    \]

\section{Proof of Theorem \ref{thm: parameter_agnostic}}
\label{sec: appendix proof}

This section collects the proof of Theorem \ref{thm: parameter_agnostic}. Overall, the proof is inspired by the proof of Theorem 4.2 in \citet{li2023convergence}, which utilizes stopping time to bound the norm of stochastic gradients.

In the following proof, we define
\small
\begin{gather}
\label{eq: def sigma}
    \sigma \overset{\triangle}{=} \max\left\{\sqrt{2R^2 \log \frac{T}{\delta}}, L\frac{\eta_t}{1-\bone}\max\{\frac{\beta_1}{\sqrt{\beta_2}}, \frac{1- \beta_1}{\sqrt{1- \beta_2}}\}, \frac{3L_0}{4L_1}\right\},\\
    \label{eq: def G}
    G \overset{\triangle}{ = } \max\{\frac{3L_0}{4L_1}, 72L_1(f(\bw_1) -f^*), \sqrt{72L_1\sigma^2\eta_t((1-\bone) T +1)},  60\sqrt{L_1R^2\sigma^2\eta_t\sqrt{2T\log(1/\delta)}}\},
    \\
    \label{eq: def F}
    F \overset{\triangle}{ = } \frac{G^2}{ 3(3L_0+4L_1 G)},
    \\
    \label{eq: def C}
    C \overset{\triangle}{ = } \sqrt{\frac{4L^2}{\varepsilon^4}(G+\sigma+\varepsilon)}.
\end{gather}
\normalsize

We consider the following stopping time:

\begin{align}
\label{eq: def stop time}
	\tau := \min\{t\mid f(\bw_t)-f^*>F\}\land \min\{t\mid \Vert{\nabla f(\bw_t)-\bg_t} \Vert>\sigma\} \land  (T+1).
\end{align}

Due to Lemma \ref{lem:reversePL} and the definition of $F$ (Eq. (\ref{eq: def F})), one can easily see that for any $t < \tau$, $\norm{\nabla f(\bw_t)} \le G$.

Also, as we are dealing with optimizers with coordinate-wise learning rates, we introduce the following norm to ease the burden of writing. Specifically, let $\boldsymbol{b}\in \mathbb{R}^d$ be a vector with each coordinate positive. For any $\boldsymbol{a} \in \mathbb{R}^d$, we define
\begin{equation*}
    \Vert \boldsymbol{a} \Vert_{\boldsymbol{b}} = \sqrt{\langle \boldsymbol{a} \odot \boldsymbol{b}, \boldsymbol{a}\rangle}.
\end{equation*}

\subsection{Useful Lemmas}

The following lemma bounds the change of $f$ through its local second-order expansion.

\begin{lemma}
\label{lem: expansion}
Let Assumption \ref{assum: objective} holds.
Then, for any three points $\bw^1, \bw^2\in \mathbb{R}^d$  satisfying $\Vert \bw^1-\bw^2\Vert \le \frac{1}{L_1}$, we have
\begin{equation*}
    f(\bw^2)\le  f(\bw^1)+\langle \nabla f(\bw^1), \bw^2-\bw^1\rangle + \frac{1}{2}(L_0+L_1 \Vert \nabla f(\bw^1)\Vert) \Vert\bw^2-\bw^1\Vert^2 
    . 
\end{equation*}
\end{lemma}
\begin{proof}
By the Fundamental Theorem of Calculus, we have
\begin{align*}
   &f(\bw^2)\\
   =& f(\bw^1)+\int_{0}^1 \langle \nabla f(\bw^1+a(\bw^2-\bw^1)), \bw^2-\bw^1\rangle
    \mathrm{d}a
    \\
    =& f(\bw^1)+\langle \nabla f(\bw^1), \bw^2-\bw^1\rangle +\int_{0}^1 \langle \nabla f(\bw^1+a(\bw^2-\bw^1))-\nabla f(\bw^1), \bw^2-\bw^1\rangle
    \mathrm{d}a
    \\
    \le & f(\bw^1)+\langle \nabla f(\bw^1), \bw^2-\bw^1\rangle +\int_{0}^1 \Vert \nabla f(\bw^1+a(\bw^2-\bw^1))-\nabla f(\bw^1)\Vert \Vert\bw^2-\bw^1\Vert
    \mathrm{d}a
    \\
    \overset{(\star)}{\le} &  f(\bw^1)+\langle \nabla f(\bw^1), \bw^2-\bw^1\rangle 
     +\int_{0}^1 (L_0+L_1 \Vert \nabla f(\bw^1)\Vert )\Vert a(\bw^2-\bw^1)
   \Vert \Vert\bw^2-\bw^1\Vert
    \mathrm{d}a
    \\
    \le & f(\bw^1)+\langle \nabla f(\bw^1), \bw^2-\bw^1\rangle + \frac{1}{2}(L_0+L_1 \Vert \nabla f(\bw^1)\Vert) \Vert\bw^2-\bw^1\Vert^2
   ,
\end{align*}
where Inequality $(\star)$ uses the fact $\Vert \bw^2 - \bw^1 \Vert \le \frac{1}{L_1}$, so that
Assumption \ref{assum: objective} can be applied.

The proof is completed.
\end{proof}

The following lemma bounds the gradient norm through the function value when Assumption \ref{assum: objective} holds.
\begin{lemma}
	\label{lem:reversePL} Under Assumptions~\ref{assum: objective}, we have $\norm{\nabla f(\bw)}^2\!\le 3(3L_0+4L_1\norm{\nabla f(\bw)})(f(\bw)-f^*)$.
\end{lemma}
\begin{proof}
	Denot $L:=3L_0+4L_1 \norm{\nabla f(\bw)}$. Let $\boldsymbol{v}:=\bw-\frac{1}{2L}\nabla f(\bw)$. Then one can easily see
	\begin{align*}
		\norm{\boldsymbol{v}-\bw} \le \frac{1}{2L_1},
	\end{align*}
    and thus Lemma \ref{lem: expansion} can be applied. Therefore, we have
	\begin{align*}
		f^*-f(\bw)\le f(\boldsymbol{v})-f(\bw)\le \langle{\nabla f(\bw)},{\boldsymbol{v}-\bw}\rangle+\frac{L}{2}\norm{\boldsymbol{v}-\bw}^2=-\frac{3 L \norm{\nabla f(\bw)}^2}{8}\le -\frac{L \norm{\nabla f(\bw)}^2}{3}.
	\end{align*}
	The proof is completed.
\end{proof}

The following lemma bounds the update of AdamS:
\begin{lemma}
\label{lem: bounded update}
For any $t$, let $\bw_t$ be the parameter of AdamS after the $t$-th iteration. Then,
\begin{equation*}
    \Vert \bw_{t+1} - \bw_t \Vert \le \eta_t\sqrt{d} \max\{\frac{\beta_1}{\sqrt{\beta_2}}, \frac{1- \beta_1}{\sqrt{1- \beta_2}}\}.
\end{equation*}
Therefore, under the hyperparameter selection of Theorem \ref{thm: parameter_agnostic}, we have $\Vert \bw_{t+1} - \bw_t \Vert = \mathcal{O}(\frac{1}{\sqrt T})$.
\end{lemma}

\begin{proof}
    We have 
    \begin{equation*}
        \Vert \bw_{t+1} - \bw_t \Vert = \eta_t \left\Vert \frac{1}{\sqrt{\bnu_t} + \varepsilon} \odot \bom_t \right\Vert = \eta_t \left\Vert \frac{1}{\sqrt{\beta_2\bom_{t-1}^{\odot 2} + (1-\beta_2) \bg_t^{\odot2}} + \varepsilon} \odot \bom_t \right\Vert. 
    \end{equation*}

    On the other hand, by Young's inequality, we have that coordinate-wisely
    \begin{equation*}
         \bom_t^{\odot 2} \le \beta_1^2  \bom_{t-1}^{\odot 2} + (1-\beta_1)^2 \bg_t^{\odot 2}.
    \end{equation*}

    The proof is completed.
\end{proof}

The following lemma bounds the adaptive conditioner $\bnu_t$.
\begin{lemma}
\label{lem: bound lr}
    If $t < \tau$, we have the $i$-th coordinate $\bnu_{t, i}$ of $\bnu_t$ satisfies
    \begin{equation*}
        0 \le \sqrt{\bnu_{t, i}} \le G+\sigma.
    \end{equation*}
\end{lemma}

\begin{proof}
    The first inequality is obvious.
    
    For the second inequality, one can easily see that $\bg_{t, i}$ satisfies the same inequality according to the definition of $\tau$. According to the definition of $\bnu_t$, we have
    \begin{equation*}
        \bnu_{t,i} = (1-\btwo) \bg_{t,i}^2 +\btwo ((1-\bone)\sum_{s=0}^{t-1} \bone^{t-1-s}\bg_{s, i})^2.
    \end{equation*}

    Applying the estimation of $\bg_{s, i}$ completes the proof.
\end{proof}

The following lemma provides a rough bound of the gap between $\nabla f(\bw_t)$ and $\bom_t$.

\begin{lemma}
	\label{lem:moment_bound} Let $\Delta_t = \bom_t -\nabla f(\bw_t)$. If $t\le \tau$, we have $\Vert {\Delta_t} \Vert \le2\sigma$.
\end{lemma}

\begin{proof} We prove this claim by induction. First, note that for $t=1$, we have $$\norm{\Delta_1}=\norm{\bg_1
-\nabla f(\bw_1)}\le\sigma\le 2\sigma.$$
	Now suppose $\norm{\Delta_t}\le 2\sigma$ for some $2\le t\le\tau$.
	According to the update rule of $\bom_t$, we have
	\begin{align*}
		\Delta_t =& \bone(\Delta_{t-1}+\nabla f(\bw_{t-1})-\nabla f(\bw_t))+(1-\bone)(\bg_t-\nabla f(\bw_t)),
	\end{align*}
	which implies 
	\begin{align*}
		\norm{\Delta_t}\le (1+\bone)\sigma + \norm{ \nabla f(\bw_{t-1})-\nabla f(\bw_t) } \le (1+\bone)\sigma +L\eta_t\max\{\frac{\beta_1}{\sqrt{\beta_2}}, \frac{1- \beta_1}{\sqrt{1- \beta_2}}\}\sqrt{d}\le 2\sigma,
	\end{align*}
    where in the second inequality, we use $\norm{\bw_{t-1} - \bw_t}\le \frac{1}{L_1}$ when $T$ is large enough and thus Assumption \ref{assum: objective} can be applied, and Lemma \ref{lem: bounded update}, and in the last inequality, we use the definition of $\sigma$ (Eq. \ref{eq: def sigma}).

As $(1-\bone)\sigma = \Theta(\log T/\sqrt{T})$, which is large than $\mathcal{O}(1/\sqrt T)$ when $T$ is large enough. The proof is completed.
\end{proof}

The following lemma bounds the gap between $\nabla f(\bw_t)$ and $\bom_t$ recursively.

\begin{lemma}
	\label{lem:sum_moment_error}
	Let $\Delta_t = \bom_t -\nabla f(\bw_t)$. With probability $1-\delta$,
	\begin{align*}
	 \sum_{t=1}^{\tau-1}\left(\frac{4(G+\sigma+\varepsilon)}{\varepsilon^2}  \norm{\Delta_{t}}^2 - \norm{\nabla f(\bw_{t})}^2\right)  \le& 4\sigma^2((1-\bone)T+1)+20R^2\sigma^2 \sqrt{2\sum_{t=2}^T\log(1/\delta)}
     \\
     =& \mathcal{O}(\sigma^2\sqrt{T\log(1/\delta)}) .
	\end{align*}
\end{lemma}
\begin{proof}
	According to the definition of $\bom_t$, we have
	\begin{align}
		\Delta_t =& \bone(\Delta_{t-1}+\nabla f(\bw_{t-1})-\nabla f(\bw_t))+(1-\bone)(g_t-\nabla f(\bw_t)).\label{eq:epsilon_t_recursive}
	\end{align}
	As $T$ is large enough, by Lemma~\ref{lem: bounded update}, we have $\Vert {\bw_{t}-\bw_{t-1}} \Vert \le \frac{1}{L_1}$. Therefore by Assumption \ref{assum: objective}, 
	\begin{align}
		\label{eq:bd1}
		\Vert{\nabla f(\bw_{t-1})-\nabla f(\bw_t)} \Vert \le L\Vert{\bw_t-\bw_{t-1}} \Vert
		\le  \frac{\eta L}{\varepsilon}\Vert{\bom_{t-1}}\Vert
		\le  \frac{\eta L}{\varepsilon}\left(\Vert{\nabla f(\bw_{t-1})}\Vert+\Vert{\Delta_{t-1}}\Vert\right),
	\end{align}
	Therefore,
	\begin{align*}
		&\Vert(\Delta_{t-1}+\nabla f(\bw_{t-1})-\nabla f(\bw_t))\Vert^2
        \\
		\le & \frac{1}{\bone}\Vert{\Delta_{t-1}} \Vert^2+\frac{1}{1-{\bone}}\Vert{\nabla f(\bw_{t-1})-\nabla f(\bw_t)}\Vert^2
        \\
		\le & \frac{1}{{\bone}}\Vert{\Delta_{t-1}} \Vert^2+\frac{1}{1-{\bone}} \frac{4\eta^2L^2}{\varepsilon^2}(\Vert{\nabla f(\bw_{t-1})}\Vert^2 + \Vert \Delta_{t-1} \Vert^2)
	\end{align*}
	where the first inequality uses Young's inequality, and the second inequality uses Eq. (\ref{eq:bd1}).

    Due to our choice of $\bone$ and $\eta$, we have $\frac{\bone^2}{1-\bone} \frac{4\eta^2L^2}{\varepsilon^2} = \mathcal{O}(1/\sqrt{T})$, which is smaller than $1-\frac{1}{2}(1-\bone)$ when $T$ is large enough. Therefore,
    \begin{align*}
         \bone^2\Vert(\Delta_{t-1}+\nabla f(\bw_{t-1})-\nabla f(\bw_t))\Vert^2
        \le  \left(\frac{1}{2}+\frac{\beta}{2} \right) \Vert \Delta_t \Vert^2 + \frac{\bone^2}{1-\bone} \frac{4\eta^2L^2}{\varepsilon^2} \Vert \nabla f(\bw_{t-1}) \Vert^2.
    \end{align*}

Therefore, applying the above inequality back to Eq. (\ref{eq:epsilon_t_recursive}), we have if $t\le \tau$,
	\begin{align}
    \nonumber
		&\Vert{\Delta_t}\Vert^2
        \\
        \nonumber
        = &\bone^2 \Vert{\Delta_{t-1}+\nabla f(\bw_{t-1})-\nabla f(\bw_t)}\Vert^2 +2\bone(1-\bone)\langle\Delta_{t-1}+\nabla f(\bw_{t-1})-\nabla f(\bw_t), g_t-\nabla f(\bw_t)\rangle
        \\
        \nonumber
        &+ (1-\bone)^2\Vert{g_t-\nabla f(\bw_t)} \Vert^2
        \\
        \nonumber
		\le & \frac{1+\bone}{2}\Vert{\Delta_{t-1}}\Vert^2+\frac{\bone^2}{1-\bone} \frac{4\eta^2L^2}{\varepsilon^2}\Vert{\nabla f(\bw_{t-1})}\Vert^2 + (1-\bone)^2\Vert{g_t-\nabla f(\bw_t)} \Vert^2
        \\
        \label{eq: delta_est_1}
        &+2\bone(1-\bone)\langle\Delta_{t-1}+\nabla f(\bw_{t-1})-\nabla f(\bw_t), g_t-\nabla f(\bw_t)\rangle,
	\end{align}
    where in the last equation we use Young's inequality.

	On the other hand, note that
	\begin{align*}
		&\bone(1-\bone)\sum_{t=2}^\tau \langle\Delta_{t-1}+\nabla f(\bw_{t-1})-\nabla f(\bw_t), g_t-\nabla f(\bw_t)\rangle
        \\
        =& \bone(1-\bone)\sum_{t=2}^T  1_{\tau\ge t}\langle\Delta_{t-1}+\nabla f(\bw_{t-1})-\nabla f(\bw_t), g_t-\nabla f(\bw_t)\rangle.
	\end{align*}
As $\E^{|\mathcal{F}_t}[1_{\tau\ge t}\langle\Delta_{t-1}+\nabla f(\bw_{t-1})-\nabla f(\bw_t), g_t-\nabla f(\bw_t)\rangle]=0$, we have that
\begin{equation*}
    V_t \overset{\triangle}{=} 1_{\tau\ge t}\langle\Delta_{t-1}+\nabla f(\bw_{t-1})-\nabla f(\bw_t), g_t-\nabla f(\bw_t)\rangle
\end{equation*}
is a martingale difference sequence. Also, according to Lemma \ref{lem:moment_bound}, we have when $T$ is large enough, $\norm{\Delta_{t-1}+\nabla f(\bw_{t-1})-\nabla f(\bw_t)} \le 3\sigma$, thus by Assumption \ref{assum: noise}, we have $V_t$ is subgaussian with constant $3\sigma R$. Then by the Azuma-Hoeffding inequality, we have with probability at least $1-\delta/2$,
	\begin{align*}
		\left\vert\sum_{t=2}^T V_t\right\vert\le 5R^2\sigma^2 \sqrt{2\sum_{t=2}^T\log(1/\delta)}.
	\end{align*}

Also, due to Assumption \ref{assum: noise}, we have with probability at least $1-\delta/2T$,
\begin{equation*}
    \Vert{g_t-\nabla f(\bw_t)} \Vert^2 \le \sqrt{2R^2\log\frac{T}{\delta}} \le \sigma.
\end{equation*}

Applying the above inequalities back to Eq. (\ref{eq: delta_est_1}), 
		\begin{align*}
			\frac{1-\beta_1}{2}\norm{\Delta_{t-1}}^2\le \frac{1-\beta_1}{2}\norm{\Delta_{t-1}}^2\le& \norm{\Delta_{t-1}}^2-\norm{\Delta_t}^2+\frac{\bone^2}{1-\bone} \frac{4\eta^2L^2}{\varepsilon^2} \Vert \nabla f(\bw_{t-1}) \Vert^2  \\&+  (1-\bone)^2 \Vert{g_t-\nabla f(\bw_t)} \Vert^2+2\bone(1-\bone)V_t.
		\end{align*}
		Taking a summation over $t$ from $2$ to $\tau$, we have with probability at least $1-\delta$,
		\begin{align*}
       & \frac{1-\bone}{2}\sum_{t=1}^{\tau-1}\left(\norm{\Delta_{t}}^2 -   \frac{\varepsilon^2}{4(G+\sigma+\varepsilon)}\norm{\nabla f(\bw_{t})}^2\right)
        \\
			\le&\sum_{t=2}^{\tau}\frac{1-\bone}{2}\norm{\Delta_{t-1}}^2-\frac{\bone^2}{1-\bone} \frac{4\eta^2L^2}{\varepsilon^2}\norm{\nabla f(\bw_{t-1})}^2
            \\
            \le & \norm{\Delta_{1}}^2-\norm{\Delta_{\tau}}^2+(1-\bone)^2 \sigma^2T+10(1-\bone)R^2\sigma^2 \sqrt{2\sum_{t=2}^T\log(1/\delta)}
            \\
            \le & 2\sigma^2((1-\bone)^2T+1)+10(1-\bone)R^2\sigma^2 \sqrt{2\sum_{t=2}^T\log(1/\delta)},
		\end{align*}
	where the first inequality is due to the assumption in Theorem \ref{thm: parameter_agnostic} that $\frac{\eta}{1-\bone} \ge C$, where $C$ is defined in Eq. (\ref{eq: def C}).

        The proof is completed.
	\end{proof}

\subsection{Proof of the full theorem}

\begin{proof}[Proof of Theorem~\ref{thm: parameter_agnostic}]
Recall that by Lemma \ref{lem: bounded update}
\begin{equation*}
    \Vert \bw_{t+1} - \bw_t \Vert  = \mathcal{O}(\frac{1}{\sqrt{T}}).
\end{equation*}
When $T$ is large enough, $\bw_t$ and $\bw_{t+1}$ will fulfill the requirement of Lemma \ref{lem: expansion}, which gives
	\begin{align*}
		f(\bw_{t+1})-f(\bw_t) \le& \langle {\nabla f(\bw_t)}, {\bw_{t+1}-\bw_t} \rangle +\frac{L_0+L_1 \Vert \nabla f(\bw_t) \Vert}{2}\Vert {\bw_{t+1}-\bw_t} \Vert^2.
	\end{align*}

If $t < \tau$, we further have $ \Vert \nabla f(\bw_t) \Vert \le G$. Therefore, if $t< \tau$, the above inequality can be further bounded by
\begin{align*}
		& f(\bw_{t+1})-f(\bw_t)
        \\
        \le& \langle {\nabla f(\bw_t)}, {\bw_{t+1}-\bw_t} \rangle +\frac{L_0+L_1 G}{2}\Vert {\bw_{t+1}-\bw_t} \Vert^2
        \\
        = & -\langle {\nabla f(\bw_t)}, \eta_t  \frac{1}{\sqrt{\bnu_t} + \varepsilon} \odot \nabla f(\bw_t) \rangle + \langle {\nabla f(\bw_t)}, \eta_t  \frac{1}{\sqrt{\bnu_t} + \varepsilon} \odot (\nabla f(\bw_t)-\bom_t) \rangle
        \\
        &+ \frac{L_0+L_1 G}{2}\eta_t^2 \left\Vert \frac{1}{\sqrt{\bnu_t} + \varepsilon} \odot \bom_t \right\Vert^2
        \\
        =& - \eta_t  \left \Vert  {\nabla f(\bw_t)} \right\Vert^2_{ \frac{1}{\sqrt{\bnu_t} + \varepsilon} } + \langle {\nabla f(\bw_t)}, \eta_t  \frac{1}{\sqrt{\bnu_t} + \varepsilon} \odot (\nabla f(\bw_t)-\bom_t) \rangle
        \\
        &+ \frac{L_0+L_1 G}{2}\eta_t^2 \left\Vert \bom_t \right\Vert^2_{\frac{1}{(\sqrt{\bnu_t} + \varepsilon)^2}}
        \\
        \overset{(\circ)}{ \le } & - \eta_t  \left \Vert  {\nabla f(\bw_t)} \right\Vert^2_{ \frac{1}{\sqrt{\bnu_t} + \varepsilon} } + \frac{1}{4}\eta_t  \left \Vert  {\nabla f(\bw_t)} \right\Vert^2_{ \frac{1}{\sqrt{\bnu_t} + \varepsilon} } + \eta_t  \left \Vert  \Delta_t \right\Vert^2_{ \frac{1}{\sqrt{\bnu_t} + \varepsilon} }
        \\
        &+ (L_0 + L_1G)\eta_t^2 \left\Vert \Delta_t \right\Vert^2_{\frac{1}{(\sqrt{\bnu_t} + \varepsilon)^2}} + (L_0 + L_1G)\eta_t^2 \left\Vert \nabla f(\bw_t) \right\Vert^2_{\frac{1}{(\sqrt{\bnu_t} + \varepsilon)^2}}
        \\
        =  &-  \frac{3}{4}\eta_t  \left \Vert  {\nabla f(\bw_t)} \right\Vert^2_{ \frac{1}{\sqrt{\bnu_t} + \varepsilon} } + \eta_t  \left \Vert  \Delta_t \right\Vert^2_{ \frac{1}{\sqrt{\bnu_t} + \varepsilon} }
        \\
        &+ (L_0 + L_1G)\eta_t^2 \left\Vert \Delta_t \right\Vert^2_{\frac{1}{(\sqrt{\bnu_t} + \varepsilon)^2}} + (L_0 + L_1G)\eta_t^2 \left\Vert \nabla f(\bw_t) \right\Vert^2_{\frac{1}{(\sqrt{\bnu_t} + \varepsilon)^2}}
\end{align*}

where $\Delta_t$ is defined as $\Delta_t = \bom_t - \nabla f(\bw_t)$ and inequality $(\circ)$ uses Young's inequality.

According to Lemma \ref{lem: bound lr}, we further have 
\begin{align*}
    & f(\bw_{t+1})-f(\bw_t)
    \\
        \le & -  \frac{3}{4}\eta_t  \left \Vert  {\nabla f(\bw_t)} \right\Vert^2_{ \frac{1}{\sqrt{\bnu_t} + \varepsilon} } + \eta_t  \left \Vert  \Delta_t \right\Vert^2_{ \frac{1}{\sqrt{\bnu_t} + \varepsilon} }
        \\
        &+ \frac{(L_0 + L_1G)\eta_t^2}{\varepsilon} \left\Vert \Delta_t \right\Vert^2_{\frac{1}{\sqrt{\bnu_t} + \varepsilon}} + \frac{(L_0 + L_1G)\eta_t^2}{\varepsilon} \left\Vert \nabla f(\bw_t) \right\Vert^2_{\frac{1}{\sqrt{\bnu_t} + \varepsilon}}.
\end{align*}

With large enough $T$, we have $\eta_t \le \frac{\varepsilon}{4(L_0+L_1 G)}$, and thus
\begin{align*}
    & f(\bw_{t+1})-f(\bw_t)
    \\
        \le & -  \frac{1}{2}\eta_t  \left \Vert  {\nabla f(\bw_t)} \right\Vert^2_{ \frac{1}{\sqrt{\bnu_t} + \varepsilon} } + 2\eta_t  \left \Vert  \Delta_t \right\Vert^2_{ \frac{1}{\sqrt{\bnu_t} + \varepsilon} }
    \\
    \le & -  \frac{1}{2(G+\sigma +\varepsilon)}\eta_t  \left \Vert  {\nabla f(\bw_t)} \right\Vert^2 + 2\frac{\eta_t}{\varepsilon}  \left \Vert  \Delta_t \right\Vert^2_{ \frac{1}{\sqrt{\bnu_t} + \varepsilon} }
    \\
     \le & -  \frac{1}{2(G+\sigma +\varepsilon)}\eta_t  \left \Vert  {\nabla f(\bw_t)} \right\Vert^2 + 2\frac{\eta_t}{\varepsilon^2}  \left \Vert  \Delta_t \right\Vert^2.
\end{align*}

After taking sum over $t$ and rearranging, we have
\begin{align}
\nonumber
	\sum_{t=1}^{\tau-1}\left( \norm{\nabla f(\bw_{t})}^2 -\frac{2(G+\sigma+\varepsilon)}{\varepsilon^2}\norm{\Delta_{t}}^2\right)\le \frac{2(G+\sigma+\varepsilon)}{\eta_t}\left(f(\bw_1)-f(\bw_\tau)\right).
	\end{align}

Multiplying both sides of the above inequality by $2$ and adding the inequality in Lemma \ref{lem:sum_moment_error}, we obtain with probability at least $1-
\delta$,
\small
\begin{align}
\label{eq: gradient sum}
	\sum_{t=1}^{\tau-1} \norm{\nabla f(\bw_{t})}^2\le & \frac{2(G+\sigma+\varepsilon)}{\eta_t}(f(\bw_1)-f(\bw_\tau)) + 4\sigma^2((1-\bone)T+1)+20R^2\sigma^2 \sqrt{2\sum_{t=2}^T\log(1/\delta)} 
    \\
    \nonumber
    = & \tilde{\mathcal{O}}(1/\sqrt{T}).
	\end{align}
    \normalsize

In the following proof, we will bound the probability of the event $\{\tau \le T\}$. Note if we can show  $\mathbb{P}(\tau > T) \ge 1-\delta$, the proof is completed, as conditional on $\{\tau> T\}$, $\sum_{t=1}^{\tau-1} \norm{\nabla f(\bw_{t})}^2$ in the above inequality will become $\sum_{t=1}^{T} \norm{\nabla f(\bw_{t})}^2$.

Obviously, the stopping time $\tau$ (eq. (\ref{eq: def stop time})) can be decomposed as $\tau:=\min\{\tau_1,\tau_2\}$, where $\tau_1$ and $\tau_2$ are two stopping times defined as 
    \begin{align*}
		\tau_1:=&\min\{t\mid f(\bw_t)-f^*>F\}\land (T+1),\\ \tau_2:=&\min\{t\mid \Vert {\nabla f(\bw_t)-\bg_t} \Vert >\sigma\}\land (T+1),
	\end{align*}
	We then bound $\mathbb{P}(\tau_1 \le T)$ and $\mathbb{P}(\tau_2 \le T)$ respectively.\\
\textbf{Bound of $\mathbb{P}(\tau_2 \le T)$.} We bound this term by a similar practice as Lemma \ref{lem:sum_moment_error}. 
According to the definition of $\tau_2$
	\begin{align*}
\Pro(\tau_2\le T)
  =& \Pro\left(\bigcup_{1\le t\le T}\left\{ \norm{\nabla f(\bw_t)-\bg_t}>\sigma \right\}\right)\\
		\le& \sum_{1\le t\le T}\Pro\left( \norm{\nabla f(\bw_t)-\bg_t}>\sigma \right)\\
		\le& 2Te^{-\frac{\sigma^2}{2R^2}}\\
		\le&\frac{\delta}{2},
	\end{align*}
	where the last inequality uses the definition of $\sigma$.

	\textbf{Bound of $\mathbb{P}(\tau_1 \le T)$.} Simple rearranging of Eq. (\ref{eq: gradient sum}) gives that, with probability $1-\frac{\delta}{2}$,
    \begin{align*}
    & \frac{2(G+\sigma+\varepsilon)}{\eta_t}(f(\bw_{\tau})-f^*)
    \\
       \le & \sum_{t=1}^{\tau-1} \norm{\nabla f(\bw_{t})}^2 + \frac{2(G+\sigma+\varepsilon)}{\eta_t}(f(\bw_{\tau})-f^*)
        \\
        \le & \frac{2(G+\sigma+\varepsilon)}{\eta_t}(f(\bw_1)-f^*) + 4\sigma^2((1-\bone)T+1)+20R^2\sigma^2 \sqrt{2\sum_{t=2}^T\log(1/\delta)}.
    \end{align*}
    Therefore, by dividing both sides of the above inequality, we obtain 
     \begin{align*}
    &f(\bw_{\tau})-f^*
        \\
        \le & (f(\bw_1)-f^*) +\frac{\eta_t}{2(G+\sigma+\varepsilon)} 4\sigma^2((1-\bone)T+1)+ \frac{\eta_t}{2(G+\sigma+\varepsilon)}20R^2\sigma^2 \sqrt{2\sum_{t=2}^T\log(1/\delta)}
        \\
        \le & \frac{G^2}{ 3(3L_0+4L_1 G)}
        \\
        = & F,
    \end{align*}
    where the last inequality uses the definition of $G$.
    
   Therefore, we have that
   \begin{align*}
       \mathbb{P}(\tau_1 \le T) \le \Pro(\text{Eq. \ref{eq: gradient sum} fails to hold}) \le \frac{\delta}{2}.
   \end{align*}
The proof is completed by $\mathbb{P}(\tau \le T) \le \mathbb{P}(\tau_1 \le T) + \mathbb{P}(\tau_2 \le T) \le \delta$.
   
\end{proof}

\end{document}